\documentclass[11pt]{article}

\usepackage{fullpage,times,url}

\usepackage{amsthm,amsfonts,amsmath,amssymb,epsfig,color,float,graphicx,verbatim}
\usepackage{algorithm,algorithmic}
\usepackage{bm}
\usepackage{hyperref}
\usepackage{enumitem}

\usepackage{authblk}

\hypersetup{
	colorlinks   = true, 
	urlcolor     = blue, 
	linkcolor    = blue, 
	citecolor   = black 
}
\usepackage{subcaption, bbm}
\newtheorem{theorem}{Theorem}[section]
\newtheorem{proposition}[theorem]{Proposition}
\newtheorem{lemma}[theorem]{Lemma}
\newtheorem{corollary}[theorem]{Corollary}

\newtheorem{assumption}[theorem]{Assumption}
\newtheorem{remark}[theorem]{Remark}

\newcommand{\reals}{\mathbb{R}}
\DeclareMathOperator*{\E}{\mathbb{E}}

\newcommand{\cd}{\mathcal{D}}

\newcommand{\be}{\mathbf{e}}
\newcommand{\bx}{\mathbf{x}}
\newcommand{\bw}{\mathbf{w}}

\newcommand{\bu}{\mathbf{u}}
\newcommand{\bv}{\mathbf{v}}

\newcommand{\by}{\mathbf{y}}

\newcommand{\Bcal}{\mathcal{B}}

\newcommand{\Dcal}{\mathcal{D}}

\newcommand{\Ncal}{\mathcal{N}}

\newcommand{\norm}[1]{\|#1\|}
\newcommand{\inner}[1]{\langle#1\rangle}

\newcommand{\secref}[1]{Section~\ref{#1}}
\newcommand{\subsecref}[1]{Subsection~\ref{#1}}
\newcommand{\figref}[1]{Figure~\ref{#1}}
\renewcommand{\eqref}[1]{Eq.~(\ref{#1})}
\newcommand{\lemref}[1]{Lemma~\ref{#1}}
\newcommand{\corollaryref}[1]{Corollary~\ref{#1}}
\newcommand{\thmref}[1]{Theorem~\ref{#1}}
\newcommand{\propref}[1]{Proposition~\ref{#1}}
\newcommand{\appref}[1]{Appendix~\ref{#1}}

\usepackage{bbm}
\usepackage[square,numbers]{natbib}

\newcommand{\zero}{{\mathbf{0}}}
\newcommand{\onefunc}{\mathbbm{1}}

\makeatletter
\newcommand{\printfnsymbol}[1]{%
  \textsuperscript{\@fnsymbol{#1}}%
}
\makeatother

\title{Learning a Single Neuron with Bias Using Gradient Descent}

\author[ ]{Gal Vardi\thanks{equal contribution}}
\author[ ]{Gilad Yehudai\printfnsymbol{1}}
\author[ ]{Ohad Shamir}
\affil[ ]{Weizmann Institute of Science}
\affil[ ]{\textit {\{gal.vardi,gilad.yehudai,ohad.shamir\}@weizmann.ac.il}}

\date{}

\date{}
\begin{document}

\maketitle

\begin{abstract}
We theoretically study the fundamental problem of learning a single neuron with a bias term ($\bx\mapsto \sigma(\inner{\bw,\bx} + b)$) in the realizable setting with the ReLU activation, using gradient descent. Perhaps surprisingly, we show that this is a significantly different and more challenging problem than the bias-less case (which was the focus of previous works on single neurons), both in terms of the optimization geometry as well as the ability of gradient methods to succeed in some scenarios. We provide a detailed study of this problem, characterizing  the critical points of the objective, demonstrating failure cases, and providing positive convergence guarantees under different sets of assumptions. To prove our results, we develop some tools which may be of independent interest, and improve previous results on learning single neurons. 
\end{abstract}

\section{Introduction}

Learning a single ReLU neuron with gradient descent is a fundamental primitive in the theory of deep learning, and has been extensively studied in recent years. Indeed, in order to understand the success of gradient descent on complicated neural networks, it seems reasonable to expect a satisfying analysis of convergence on a single neuron. Although many previous works studied the problem of learning a single neuron with gradient descent, none of them considered this problem with an explicit bias term.

In this work, we study the common setting of learning a single neuron with respect to the squared loss, using gradient descent. We focus on the realizable setting, where the inputs are drawn from a distribution $\Dcal$ on $\reals^{d+1}$, and are labeled by a single target neuron of the form $\bx\mapsto \sigma(\inner{\bv,\bx})$, where $\sigma: \reals \to \reals$ is some non-linear activation function. To capture the bias term, we assume that the distribution $\Dcal$ is such that its first $d$ components are drawn from some distribution $\tilde{\Dcal}$ on $\reals^d$, and the last component is a constant $1$. Thus, the input $\bx$ can be decomposed as $(\tilde{\bx},1)$ with $\tilde{\bx}\sim\tilde{\Dcal}$, the vector $\bv$ can be decomposed as $(\tilde{\bv},b_\bv)$, where $\tilde{\bv}\in \reals^d$ and $b_{\bv}\in \reals$, and the target neuron computes a function of the form $\bx\mapsto \sigma(\inner{\tilde{\bv},\tilde{\bx}}+b_{\bv})$. Similarly, we can define the learned neuron as $\bx\mapsto \sigma(\inner{\tilde{\bw},\tilde{\bx}}+b_{\bw})$, where $\bw=(\tilde{\bw},b_{\bw})$. Overall, we can write the objective function we wish to optimize as follows:
\begin{align}
F(\bw) ~&:=~ 
\E_{\bx\sim\Dcal}\left[\frac{1}{2}\left(\sigma(\bw^\top\bx)-\sigma(\bv^\top\bx)\right)^2\right]\label{eq:single neuron}\\
&~=~\E_{\tilde{\bx}\sim\tilde{\Dcal}}\left[\frac{1}{2}\left(\sigma(\tilde{\bw}^\top\tilde{\bx} + b_\bw)-\sigma(\tilde{\bv}^\top\tilde{\bx} + b_\bv)\right)^2\right].\label{eq:single neuron with bias}
\end{align}
Throughout the paper we consider the commonly used ReLU activation function: $\sigma(x)=\max\{0,x\}$.

Although the problem of learning a single neuron is well studied  (e.g. \cite{soltanolkotabi2017learning, yehudai2020learning, frei2020agnostic, du2017convolutional, kalan2019fitting, tan2019online, mei2016landscape, oymak2018overparameterized}), none of the previous works considered the problem with an additional bias term. Moreover, previous works on learning a single neuron with gradient methods have certain assumptions on the input distribution $\Dcal$, which do not apply when dealing with a bias term (for example, a certain "spread" in all directions, which does not apply when $\Dcal$ is supported on $\{1\}$ in the last coordinate).

Since neural networks with bias terms are the common practice, it is natural to ask how adding a bias term affects the optimization landscape and the convergence of gradient descent. Although one might conjecture that this is just a small modification to the problem, we in fact show that the effect of adding a bias term is very significant, both in terms of the optimization landscape and in terms of which gradient descent strategies can or cannot work. Our main contributions are as follows:
\begin{itemize}[leftmargin=*]
\item We start in Section~\ref{sec:negative results} with some negative results, which demonstrate how adding a bias term makes the problem more difficult. In particular, we show that with a bias term, gradient descent or gradient flow\footnote{I.e., gradient descent with infinitesimal step size.} can sometimes fail with probability close to half over the initialization, even when the input distribution is uniform over a ball. In contrast,  \cite{yehudai2020learning} show that without a bias term, for the same input distribution, gradient flow converges to the global minimum with probability $1$.


\item In Section~\ref{sec:characterization} we give a full characterization of the critical points of the loss function. We show that adding a bias term changes the optimization landscape significantly: In previous works (cf.  \cite{yehudai2020learning}) it has been shown that under mild assumptions on the input distribution, the only critical points are $\bw=\bv$ (i.e., the global minimum) and $\bw=\zero$. We prove that when we have a bias term, the set of critical points has a positive measure, and that there is a cone of local minima where the loss function is flat.

\item 
	In Sections~\ref{sec:init better than trivial} and~\ref{sec: symmetric dist} we show that gradient descent converges to the global minimum at a linear rate, under some assumptions on the input distribution and on the initialization. We give two positive convergence results, where each result is under different assumptions,  
and thus the results complement each other. 
We also use different techniques for proving each of the results: The analysis in Section~\ref{sec: symmetric dist} follows from some geometric arguments and extends the technique from  \cite{yehudai2020learning, frei2020agnostic}. The analysis in Section~\ref{sec:init better than trivial} introduces a novel technique, not used in previous works on learning a single neuron, and has a more algebraic nature. Moreover, that analysis implies that under mild assumptions, gradient descent with random initialization converges to the global minimum with probability $1-e^{\Omega(d)}$.

\item The best known result for learning a single neuron without bias using gradient descent for an input distribution that is not spherically symmetric, establishes convergence to the global minimum with probability 
close to $\frac{1}{2}$
over the random initialization \citep{yehudai2020learning, frei2020agnostic}. With our novel proof technique presented in \secref{sec:init better than trivial} this result can be improved to probability at least $1-e^{-\Omega(d)}$ (see Remark~\ref{rem:improve no bias}). 

\end{itemize}

\section*{Related work}

Although there are no previous works on learning a single neuron with an explicit bias term, there are many works that consider the problem of a single neuron under different settings and assumptions.

Several papers showed that the problem of learning a single neuron can be solved under minimal assumptions using algorithms which are not gradient-based (such as gradient descent or SGD). These algorithms include the Isotron proposed by \cite{kalai2009isotron} and the GLMtron proposed by \cite{kakade2011efficient}. The GLMtron algorithm is also analyzed in \cite{diakonikolas2020approximation}. These algorithms allow learning a single neuron with bias. We note that these are non-standard algorithms, whereas we focus on the standard gradient descent algorithm.
An efficient algorithm for learning a single neuron with error parameter $\epsilon = \Omega(1/\log(d))$ was also obtained in \cite{goel2017reliably}.

In \cite{mei2016landscape} the authors study the empirical risk of the single neuron problem. However, their analysis does not include the ReLU activation, or adding a bias term. A related analysis is also given in \cite{oymak2018overparameterized}, where the ReLU activation is not considered.

Several papers showed convergence guarantees for the single neuron problem with ReLU activation under certain distributional assumptions, although none of these assumptions allows for a bias term. Notably, \cite{tian2017analytical,soltanolkotabi2017learning,kalan2019fitting,brutzkus2017globally} showed convergence guarantees for gradient methods when the inputs have a standard Gaussian distribution, without a bias term. \cite{du2017convolutional} showed that under a certain subspace eigenvalue assumption a single neuron can be learned with SGD, although this assumption does not allow adding a bias term. \cite{yehudai2020learning,frei2020agnostic} use an assumption about the input distribution being sufficiently "spread" in all directions, which does not allow for a bias term (since that requires an input distribution supported on $\{1\}$ in the last coordinate). \cite{yehudai2020learning} showed a convergence result under the realizable setting, while \cite{frei2020agnostic} considered the agnostic and noisy settings. In \cite{tan2019online} convergence guarantees are given for the absolute value activation, and a specific distribution which does not allow a bias term.

Less directly related, \cite{vardi2020implicit} studied the problem of implicit regularization in the single neuron setting. In  \cite{yehudai2019power,kamath2020approximate} it is shown that approximating a single neuron using random features (or kernel methods) is not tractable in high dimensions. We note that these results explicitly require that the single neuron which is being approximated will have a bias term. Thus, our work complements these works by showing that the problem of learning a single neuron with bias is also learnable using gradient descent (under certain assumptions). Agnostically learning a single neuron with non gradient-based algorithms and hardness of a agnostically learning a single neuron were studied in \cite{diakonikolas2020approximation,goel2019time,goel2020tight}.

\section{Preliminaries}

\paragraph{Notations.}
We use bold-faced letters to denote vectors, e.g., $\bx=(x_1,\ldots,x_d)$. For $\bu \in \reals^d$ we denote by $\norm{\bu}$ the Euclidean norm. We denote $\bar{\bu} = \frac{\bu}{\norm{\bu}}$, namely, the unit vector in the direction of $\bu$. For $1 \leq i \leq j \leq d$ we denote $\bu_{i:j}=(u_i,\ldots,u_j) \in \reals^{j-i+1}$.
We denote by $\onefunc(\cdot)$ the indicator function, for example $\onefunc(t \geq 5)$ equals $1$ if $t \geq 5$ and $0$ otherwise.
We denote by $U([-r,r])$ the uniform distribution over the interval $[-r,r]$ in $\reals$, and by $\mathcal{N}(\zero,\Sigma)$ the multivariate normal distribution with mean $\zero$ and covariance matrix $\Sigma$.
Given two vectors $\bw,\bv$ we let $\theta(\bw,\bv) = \arccos \left( \frac{\inner{\bw,\bv}}{\norm{\bw} \norm{\bv}} \right) =  \arccos(\inner{\bar{\bw},\bar{\bv}}) \in [0,\pi]$. For a vector $\bu \in \reals^{d+1}$ we often denote by $\tilde{\bu} \in \reals^d$ the first $d$ components of $\bu$, and denote by $b_\bu \in \reals$ its last component.

\paragraph{Gradient methods.}
In this paper we focus on the following two standard gradient methods for optimizing our objective $F(\bw)$ from \eqref{eq:single neuron with bias}:
\begin{itemize}[leftmargin=*]
    \item \textbf{Gradient descent:} We initialize at some $\bw_0 \in \reals^{d+1}$, and set a fixed learning rate $\eta >0$. At each iteration $t \geq 0$ we 
    have:
    $\bw_{t+1} = \bw_t - \eta \nabla F(\bw_t)$.
    
    \item \textbf{Gradient Flow:} We initialize at some $\bw(0) \in \reals^{d+1}$, and for every time $t \geq 0$, we set $\bw(t)$ to be the solution of the differential equation $\Dot{\bw} = -\nabla F(\bw(t))$. This can be thought of as a continuous form of gradient descent, where the learning rate is infinitesimally small.
\end{itemize}
The gradient of the objective in \eqref{eq:single neuron} is:
\begin{equation}\label{eq:gradient of a single neuron}
\nabla F(\bw) = 
\E_{\bx\sim\Dcal}\left[\left(\sigma(\bw^\top\bx)-\sigma(\bv^\top\bx)\right)
\cdot\sigma'(\bw^\top\bx)\bx\right].
\end{equation}
Since $\sigma$ is the ReLU function, it is differentiable everywhere except for $0$. Practical implementations of gradient methods define $\sigma'(0)$ to be some constant in $[0,1]$. Following this convention, the gradient used by these methods still correspond to \eqref{eq:gradient of a single neuron}. We note that the exact value of $\sigma'(0)$ has no effect on our results.

\section{Negative results}
\label{sec:negative results}

In this section we demonstrate that adding bias to the problem of learning a single neuron with gradient descent can make the problem significantly harder. 

First, on an intuitive level, previous results (e.g., \cite{yehudai2020learning, frei2020agnostic, soltanolkotabi2017learning, du2017convolutional, tan2019online}) considered assumptions on the input distribution, which 
require
enough "spread" in all directions (for example, a strictly positive density in some neighborhood around the origin). 
Adding a bias term, even if the first $d$ coordinates of the distribution satisfy a "spread" assumption, will give rise to a direction without "spread", since in this direction the distribution is concentrated on $1$, hence the previous results do not apply.

Next, we show two negative results where the input distribution is uniform on a ball around the origin. We note that due to Theorem 6.4 from \cite{yehudai2020learning}, we know that gradient flow on a single neuron without bias will converge to the global minimum with probability $1$ over the random initialization. The only case where 
it will fail to converge 
is when $\bw_0$ is initialized in the exact direction $-\bv$, which happens with probability $0$ with standard random initializations.

\subsection{Initialization in a flat region}

If we initialize the bias term in the same manner as the other coordinates, then we can show that gradient descent will fail with probability close to half, even if the input distribution is uniform over a (certain) origin-centered ball:
\begin{theorem}\label{thm:negative example 1}
Suppose we initialize each coordinate of $\bw_0$ (including the bias) according to $U([-1,1])$. Let $\epsilon >0$ and 
let $\tilde{\cd}$ be the uniform distribution supported 
on a ball around the origin in $\reals^d$ of radius $\epsilon$. Then,  w.p $> 1/2 - \epsilon\sqrt{d}$, gradient descent on the objective in 
\eqref{eq:single neuron with bias}
satisfies $\bw_t=\bw_0$ for all $t$ (namely, it gets stuck at its initial point $\bw_0$). 
\end{theorem}

Note that by Theorem 6.4 in \cite{yehudai2020learning}, if there is no bias term in the objective, then gradient descent will converge to the global minimum w.p $1$ using this random initialization scheme and this input distribution.
The intuition for the proof is that with constant probability over the initialization, $b_\bw$ is small enough so that $\sigma(\tilde{\bw}^\top \tilde{\bx} + b_\bw) = 0 $ almost surely. If this happens, then the gradient will 
be $\zero$ and gradient descent will never move.
The full proof can be found in 
\appref{appen:proofs from neg result1}.
We note that \thmref{thm:negative example 1} is applicable when $\epsilon$ is sufficiently small, e.g. $\epsilon \ll 1/\sqrt{d}$.



\subsection{Targets with negative bias}
\label{sec:negative2}

\thmref{thm:negative example 1} shows a difference between learning with and without the bias term. A main caveat of this example is the requirement that the bias is initialized in the same manner as the other parameters. Standard deep learning libraries (e.g. Pytorch \citep{paszke2019pytorch}) often initialize the bias term to zero by default, while using random initialization schemes for the other parameters. 

Alas, we now show that even if we initialize the bias term to be exactly zero, and the input distribution is uniform over an arbitrary origin-centered ball, we might fail to converge to the global minimum for certain target neurons:


\begin{theorem}
\label{thm:negative2}
	Let $\tilde{\cd}$ be the uniform distribution on $\Bcal = \{\tilde{\bx} \in \reals^{d}: \norm{\tilde{\bx}} \leq r\}$ for some $r>0$.
	Let $\bv \in \reals^{d+1}$ such that $\tilde{\bv} = (1,0,\ldots,0)^\top$ and $b_\bv = - \left(r - \frac{r}{2d^2}\right)$.
	Let $\bw_0 \in \reals^{d+1}$  such that $b_{\bw_0} = 0$ and $\tilde{\bw}_0$ is drawn from the uniform distribution on a sphere of radius $\rho>0$.
	Then, with probability at least $\frac{1}{2} - o_d(1)$ over the choice of $\bw_0$, gradient flow does not converge to the global minimum.
\end{theorem}

We prove the theorem in Appendix~\ref{appen:proofs from neg result2}. The intuition behind the proof is the following: The target neuron has a large negative bias, so that only a small (but positive) measure of input points are labelled as non-zero. By randomly initializing $\tilde{\bw}$, with probability close to $\frac{1}{2}$ there are no inputs that both $\bv$ and $\bw$ label positively. 
Since the gradient is affected only by inputs that $\bw$ labels positively, then during the optimization process the gradient will be independent of the direction of $\bv$, and $\bw$ will not converge to the global minimum. 


\begin{remark}\label{rem:negative bias}
	\thmref{thm:negative2} shows that gradient flow is not guaranteed to converge to a global minimum when $b_\bv$ is negative, instead it converges to a local minimum with a loss of $F(\bm{0})$. However, the loss $F(\bm{0})$ is determined by the input distribution. Take $\bv = \left(1,0,\ldots,0,- \left(r - \frac{r}{2d^2}\right)\right)^\top$ considered in the theorem. On one hand, for a uniform distribution on a ball of radius $r$ as in the theorem we have:
	\begin{align*}
		F(\zero) 
		&= \frac{1}{2} \cdot \E_\bx \left[ \left( \sigma(\bv^\top \bx) \right)^2\right]
		= \frac{1}{2} \cdot \E_\bx \left[ \onefunc(\bv^\top \bx \geq 0) \left(\bv^\top \bx \right)^2 \right]
		\\
		&= \frac{1}{2} \cdot \E_\bx \left[ \onefunc \left(x_1 \geq r - \frac{r}{2d^2} \right) \left(x_1 -  \left(r - \frac{r}{2d^2}  \right)\right)^2 \right]
		\\
		&\leq \frac{1}{2} \cdot \frac{r^2}{4d^4} \cdot \Pr_\bx \left[ x_1 \geq  r \left(1 - \frac{1}{2d^2} \right) \right]
		\leq r^2 e^{-\Omega(d)}~.
	\end{align*}
	Thus, for any reasonable $r$, a local minimum with loss $F(\zero)$ is almost as good as the global minimum. On the other hand, take a distribution $\tilde{\Dcal}$ with a support bounded in a ball of radius $r$, such that half of its mass is uniformly distributed in $A:=\left\{\tilde{\bx}\in\reals^d:x_1 >r- \frac{r}{4d^2}\right\}$, and the other half is uniformly distributed in $\Bcal \setminus A$. In this case, it is not hard to see that the same proof as in \thmref{thm:negative2} works, and gradient flow will converge to a local minimum with loss $F(\bm{0}) = \Omega\left(\frac{r}{d^2}\right)$, which is arbitrarily large if $r$ is large enough.
\end{remark}

Although in the example given in  \thmref{thm:negative2} the objective at $\bw=\bm{0}$ is almost as good as the objective at $\bw=\bv$, we emphasize that w.p almost $\frac{1}{2}$ gradient flow cannot reach the global minimum even asymptotically. On the other hand, in the bias-less case by Theorem 6.4 in \cite{yehudai2020learning} gradient flow on the same input distribution will reach the global minimum w.p $1$. Also note that the scale of the initialization of $\bw_0$ has no effect on the result.

\section{Characterization of the critical points}\label{sec:characterization}

In the previous section we have shown two examples where gradient methods on the problem of a single neuron with bias will either get stuck in a flat region, or converge to a local minimum. In this section we delve deeper into the 
examples
presented in the previous section, and give a full characterization of the critical points of the objective. We will use the following assumption on the input distribution:

\begin{assumption}\label{assum:characterization}
    The
    distribution $\tilde{\Dcal}$ on $\reals^d$
    has a density function $p(\tilde{\bx})$, and
    there are $\beta,c >0$, such that $\tilde{\Dcal}$ is supported on $\{\tilde{\bx}:\norm{\tilde{\bx}} \leq c\}$, and
    for every $\tilde{\bx}$ in the support we have $p(\tilde{\bx}) \geq \beta$.
\end{assumption}

The assumption essentially states that the distribution over the first $d$ coordinates (\emph{without} the bias term) has enough "spread" in all directions, and covers standard distributions such as uniform over a ball of radius $c$. Other similar assumptions are made in previous works (e.g. \cite{yehudai2020learning,frei2020agnostic}). We note that in \cite{yehudai2020learning} it is shown that without any assumption on the distribution, it is impossible to ensure convergence, hence we must have some kind of assumption for this problem to be learnable with gradient methods. Under this assumption we can characterize the critical points of the objective. 



\begin{theorem}\label{thm:characterization}
Consider the objective in \eqref{eq:single neuron with bias} with $\bv\neq \bm{0}$, and assume that the distribution $\tilde{\Dcal}$ on the first $d$ coordinates satisfies Assumption \ref{assum:characterization}.
Then $\bw\neq \mathbf{0}$ is a critical point of $F$ (i.e., is a root of \eqref{eq:gradient of a single neuron}) if and only if it satisfies one of the following:
\begin{itemize}[leftmargin=*]
    \item $\bw=\bv$, in which case $\bw$ is a global minimum.
    \item $\bw=(\tilde{\bw},b_{\bw})$ where $\tilde{\bw} = \zero$ and $b_\bw <0$.
    \item $\tilde{\bw}\neq 0$ and $-\frac{b_\bw}{\norm{\tilde{\bw}}} \geq c$.
\end{itemize}
In the latter two cases, $F(\bw)=F(\bm{0})$. Hence, if $\mathbf{0}$ is not a global minimum, then $\bw$ is not a global minimum. 
\end{theorem}
We note that $F(\mathbf{0})=\frac{1}{2}\E_{\bx}[\sigma(\bv^\top \bx)^2]$, so $\mathbf{0}$ is a global minimum only if the target neuron returns $0$ with probability $1$.


\begin{remark}[The case $\bw=\mathbf{0}$] 
We intentionally avoided characterizing the point $\bw=\bm{0}$, since the objective is not differentiable there (this is the only point of non-differentiability), and the gradient there is determined by the value of the ReLU activation at $0$. For $\sigma'(0)=0$ the gradient at $\bw=\bm{0}$ is zero, and this is a non-differentiable saddle point. For $\sigma'(0)=1$ (or any other positive value), the gradient at $\bw=\bm{0}$ is non-zero, and it will point at a direction which depends on the distribution. We note that in \cite{soltanolkotabi2017learning} the authors define $\sigma'(0)=1$, and use a symmetric distribution, in which case the gradient at $\bw=\bm{0}$ points exactly at the direction of the target $\bv$. This is a crucial part of their convergence analysis.
\end{remark}

We emphasize that with a bias term, there is a non-zero measure manifold of critical points (corresponding to the third bullet in the theorem). On the other hand, without a bias term the only critical point besides the global minimum (under mild assumptions on the input distribution) is at the origin $\bw = \zero$ (cf. \cite{yehudai2020learning}). The full proof is in \appref{appen:proofs from charac}.

The assumption on the support of $\tilde{\cd}$ is made for simplicity. It can be relaxed to having a distribution with exponentially bounded tail, e.g. standard Gaussian. In this case, some of the critical points will instead have a non-zero gradient which is exponentially small. We emphasize that when running optimization algorithms on finite-precision machines, which are used in practice, these "almost" critical points behave essentially like critical points since the gradient is extremely small.


Revisiting the negative examples from \secref{sec:negative results}, the first example (\thmref{thm:negative example 1}) shows that if we do not initialize the bias of $\bw$ to zero, then there is a positive probability to initialize at a critical point which is not the global minimum. The second example (\thmref{thm:negative2}) shows that even if we initialize the bias of $\bw$ to be zero, there is still a positive probability to converge to a critical point which is not the global minimum. Hence, in order to guarantee convergence we need to have more assumptions on either the input distribution, the target $\bv$ or the initialization. In the next section, we show that adding such assumptions are indeed sufficient to get positive convergence guarantees. 

\section{Convergence for initialization with loss slightly better than trivial}
\label{sec:init better than trivial}

In this section, we show that under some assumptions on the input distribution, if gradient descent is initialized such that  $F(\bw_0) < F(\zero)$ then it is guaranteed to converge to the global minimum. In \subsecref{subsec:random init}, we study under what conditions this is likely to occur with standard random initialization. 

\subsection{Convergence if $F(\bw_0)<F(\mathbf{0})$} \label{subsec:init better}

To state our results, we need the following assumption:
\begin{assumption}
\label{ass:init better than trivial}
	\;
	\begin{enumerate}[leftmargin=*]
		\item The distribution $\cd$ is supported on $\{\bx \in \reals^{d+1}: \norm{\bx} \leq c\}$ for some $c \geq 1$. 
		\item  The distribution $\tilde{\cd}$ over the first $d$ coordinates is bounded in all directions: there is $c' > 0$ such that for every $\tilde{\bu}$ with $\norm{\tilde{\bu}}=1$ and every $a \in \reals$ and $b \geq 0$, we have $\Pr_{\tilde{\bx} \sim \tilde{\cd}}\left[\tilde{\bu}^\top \tilde{\bx} \in [a,a+b] \right] \leq b \cdot c'$.
		\item We assume w.l.o.g. that $\norm{\bv} = 1$ and $c' \geq 1$.
	\end{enumerate}
\end{assumption}

Assumption (3)  
helps simplifying some expressions in our convergence result, and is not necessary.
Assumption (2) requires that the distribution is not too concentrated in a short interval. For example, if $\tilde{\cd}$ is spherically symmetric then the marginal density of the first (or any other) coordinate is bounded by $c'$.
Note that we do not assume that $\tilde{\cd}$ is spherically symmetric. 

\begin{theorem}
\label{thm:smooth linear rate}
	Under Assumption~\ref{ass:init better than trivial} we have the following.
	Let $\delta>0$ and let $\bw_0 \in \reals^{d+1}$ such that $F(\bw_0) \leq F(\zero) - \delta$. 
	Let $\gamma =  \frac{\delta^3}{3 \cdot 12^2(\norm{\bw_0}+2)^3c^8c'^2}$. 
	Assume that gradient descent runs starting from $\bw_0$ with step size $\eta \leq \frac{\gamma}{c^4}$. Then, for every $t$ we have 
	\[
		\norm{\bw_t - \bv}^2 \leq \norm{\bw_0 - \bv}^2 \left(1-\gamma \eta \right)^t~.
	\]
\end{theorem}

The formal proof appears in Appendix~\ref{appen:proofs from init better than trivial}, but we provide the main ideas below.
First, note that
\begin{align*}
	\norm{\bw_{t+1}-\bv}^2 
	&= \norm{\bw_t - \eta \nabla F(\bw_t) -\bv}^2 
	\\
	&=\norm{\bw_t - \bv}^2 - 2 \eta \inner{\nabla F(\bw_t), \bw_t - \bv} + \eta^2 \norm{\nabla F(\bw_t)}^2~.
\end{align*}
Hence, in order to show that $\norm{\bw_{t+1}-\bv}^2 \leq \norm{\bw_t - \bv}^2 \left(1-\gamma \eta \right)$ we need to obtain an upper bound for $\norm{\nabla F(\bw_t)}$ and a lower bound for $\inner{\nabla F(\bw_t), \bw_t - \bv}$. Achieving the lower bound for $\inner{\nabla F(\bw_t), \bw_t - \bv}$ is the challenging part, and we show that in order to establish such a bound it suffices to obtain a lower bound for $\Pr_\bx \left[ \bw_t^\top \bx \geq 0, \bv^\top \bx \geq 0 \right]$. We prove that if $F(\bw_t) \leq F(\zero)-\delta$ then $\Pr_\bx \left[ \bw_t^\top \bx \geq 0, \bv^\top \bx \geq 0 \right] \geq \frac{\delta}{c^2 \norm{\bw_t}}$.
Hence, if $F(\bw_t)$ remains at most $F(\zero)-\delta$ for every $t$, then a lower bound for  $\inner{\nabla F(\bw_t), \bw_t - \bv}$ can be achieved, which completes the proof. However, it is not obvious that $F(\bw_t)$ remains at most $F(\zero)-\delta$ throughout the training process. When running gradient descent on a smooth loss function we can choose a sufficiently small step size such that the loss decreases in each step, but here the function $F(\bw)$ is highly non-smooth around $\bw=\zero$. That is, the Lipschitz constant of $\nabla F(\bw)$ is unbounded. We show that if $F(\bw_t) \leq F(\zero)-\delta$ then $\bw_t$ is sufficiently far from $\zero$, and hence the smoothness of $F$ around $\bw_t$ can be bounded, which allows us to choose a small step size that ensures that $F(\bw_{t+1}) \leq F(\bw_t) \leq F(\zero)-\delta$. Hence, it follows that $F(\bw_t)$ remains at most $F(\zero)-\delta$ for every $t$.

As an aside, recall that in Section~\ref{sec:characterization} we showed that other than $\bw=\bv$ all critical points of $F(\bw)$ are in a flat region where $F(\bw)=F(\zero)$. Hence, the fact that $F(\bw_t)$ remains at most $F(\zero)-\delta$ for every $t$ implies that $\bw_t$ does not reach the region of bad critical points, which explains the asymptotic convergence to the global minimum.

We also note that although we assume that the distribution has a bounded support, this assumption is mainly made for simplicity, and can be relaxed to have sub-Gaussian distributions with bounded moments. These distributions include, e.g. Gaussian distributions. 

\subsection{Convergence for Random Initialization} \label{subsec:random init}

In Theorem~\ref{thm:smooth linear rate} we showed that if $F(\bw_0) < F(\zero)$ then gradient descent converges to the global minimum. We now show that under mild assumptions on the input distribution, a random initialization of $\bw_0$ near zero satisfies this requirement. We will need the following assumption, also used in \cite{yehudai2020learning, frei2020agnostic}:

\begin{assumption}\label{assum:spread alpha beta}
There are $\alpha,\beta > 0$ s.t the distribution $\tilde{\Dcal}$ satisfies the following: For any vector $\tilde{\bw} \neq \tilde{\bv}$, let $\tilde{\Dcal}_{\tilde{\bw},\tilde{\bv}}$ 
	denote the marginal 
	distribution of 
	$\tilde{\cd}$
	on the subspace spanned by $\tilde{\bw},\tilde{\bv}$ (as a 
	distribution 
	over $\reals^2$). Then any such distribution has a density function 
	$p_{\tilde{\bw},\tilde{\bv}}(\hat{\bx})$ over $\reals^2$ such that $\inf_{\hat{\bx}:\norm{\hat{\bx}}\leq 
		\alpha}p_{\tilde{\bw},\tilde{\bv}}(\hat{\bx})\geq \beta$. 
\end{assumption}

The main technical tool for proving convergence under random initialization is the following:
\begin{theorem}
\label{thm:random init}
	Assume that the input distribution $\cd$ is supported on $\{\bx \in \reals^{d+1}: \norm{\bx} \leq c\}$ for some $c \geq 1$, and Assumption \ref{assum:spread alpha beta} holds.
	Let $\bv \in \reals^{d+1}$ such that $\norm{\bv}=1$ and 
	$-\frac{b_\bv}{\norm{\tilde{\bv}}} \leq \alpha \cdot \frac{\sin\left(\frac{\pi}{8}\right)}{4}$.
	Let $M = \frac{\alpha^4 \beta \sin^3 \left(\frac{\pi}{8}\right)}{256c}$.
	Let $\bw \in \reals^{d+1}$ such that $b_\bw = 0$, $\theta(\tilde{\bw},\tilde{\bv}) \leq \frac{3\pi}{4}$ and $\norm{\bw} < \frac{2M}{c^2}$. 
	Then, $F(\bw) \leq F(\zero) +  \norm{\bw}^2 \cdot \frac{c^2}{2}  -  \norm{\bw} \cdot M < F(\zero)$.
\end{theorem}

We prove the theorem in Appendix~\ref{appen:proofs from random init}. The main idea is that since
\begin{align*}
	F(\bw)
	&=\E_\bx \left[  \frac{1}{2} \left(\sigma(\bw^\top \bx) - \sigma(\bv^\top \bx)\right)^2 \right]
	\\
	&= F(\zero) +  \frac{1}{2}\E_\bx \left[ \left(\sigma(\bw^\top \bx)\right)^2 \right] - \E_\bx \left[ \sigma(\bw^\top \bx) \sigma(\bv^\top \bx) \right]
	\\
	&\leq  F(\zero) +  \norm{\bw}^2 \cdot \frac{c^2}{2} - \norm{\bw} \cdot \E_\bx \left[ \sigma(\bar{\bw}^\top \bx) \sigma(\bv^\top \bx) \right]~,
\end{align*}
then it suffices to obtain a lower bound for $\E_\bx \left[ \sigma(\bar{\bw}^\top \bx) \sigma(\bv^\top \bx) \right]$. In the proof we show that such a bound can be achieved if the conditions of the theorem hold.

Suppose that $\bw_0$ is such that $\tilde{\bw}_0$ is drawn from a spherically symmetric distribution and $b_{\bw_0}=0$. By standard concentration of measure arguments, it holds w.p. at least $1-e^{\Omega(d)}$ that  $\theta(\tilde{\bw}_0,\tilde{\bv}) \leq \frac{3\pi}{4}$ (where the notation $\Omega(d)$ hides only numerical constants, namely, it does not depend on other parameters of the problem). 
Therefore, if $\tilde{\bw}_0$ is drawn from the uniform distribution on a sphere of radius $\rho < \frac{2M}{c^2}$, then the theorem implies that w.h.p. we have $F(\bw_0)<F(\zero)$. For such initialization Theorem~\ref{thm:smooth linear rate} implies that gradient descent converges to the global minimum. 
For example, for $\rho = \frac{M}{c^2}$ we have w.h.p. that  $F(\bw_0) \leq F(\zero) + \frac{\rho^2 c^2}{2} - \rho M = F(\zero) - \frac{M^2}{2c^2}$, and thus Theorem~\ref{thm:smooth linear rate} applies with $\delta = \frac{M^2}{2c^2}$. Thus, we have the following corollary:

\begin{corollary}\label{cor:rand int conv}
Under Assumption~\ref{ass:init better than trivial} and  Assumption~\ref{assum:spread alpha beta} we have the following. 
Let $M = \frac{\alpha^4 \beta \sin^3 \left(\frac{\pi}{8}\right)}{256c}$, let $\rho = \frac{M}{c^2}$, let $\delta = \frac{M^2}{2c^2}$, and let $\gamma =  \frac{\delta^3}{3 \cdot 12^2(\rho+2)^3c^8c'^2}$. 
Suppose that 
$-\frac{b_\bv}{\norm{\tilde{\bv}}} \leq \alpha \cdot \frac{\sin\left(\frac{\pi}{8}\right)}{4}$,
and $\bw_0$ is such that $b_{\bw_0}=0$ and $\tilde{\bw}_0$ is drawn from the uniform distribution on a sphere of radius $\rho$. Consider gradient descent with step size $\eta \leq \frac{\gamma}{c^4}$. Then, with probability at least $1-e^{\Omega(d)}$ over the choice of $\bw_0$ we have for every $t$: $\norm{\bw_t - \bv}^2 \leq \norm{\bw_0 - \bv}^2 \left(1-\gamma \eta \right)^t~$.
\end{corollary}
We note that a similar result holds also if $\tilde{\bw}_0$ is drawn from a normal distribution $\Ncal(\zero,\frac{\rho^2}{d} I)$.



\begin{remark}[The assumption on $b_\bv$]
The assumption $-\frac{b_\bv}{\norm{\tilde{\bv}}} \leq \alpha \cdot \frac{\sin\left(\frac{\pi}{8}\right)}{4}$ implies that the bias term $b_\bv$ may be either positive or negative, but in case it is negative then it cannot be too large. This assumption is indeed crucial for the proof, but for "well-behaved" distributions, if this assumption is not satisfied (for a large enough $\alpha$), then the loss at $F(\bm{0})$ is already good enough. For example, for a standard Gaussian distribution and for every $\epsilon >0$, we can choose $\alpha$ large enough such that for any bias term (positive or negative) we either: (1) converge to the global minimum with a loss of zero, or; (2) converge to a local minimum with a loss of $F(\bm{0})$, which is smaller then $\epsilon$. Moreover, we can show that by choosing $\alpha$ appropriately, and using the example in \thmref{thm:negative2}, if $-\frac{b_\bv}{\norm{\tilde{\bv}}} \geq 2\alpha$ then gradient flow will converge to a non-global minimum with loss of $F(\bm{0})$. This means that our bound on $\alpha$ is tight up to a constant factor. For a further discussion on the assumption on $b_\bv$, and how to choose $\alpha$ see \appref{appen:disc bv}.
\end{remark}

Previous papers have shown separation between random features (or kernel) methods and neural networks in terms of their approximation power (see \cite{yehudai2019power,kamath2020approximate}, and the discussion in \cite{malach2021quantifying}). These works show that under a standard Gaussian distribution, random features cannot even approximate a single ReLU neuron, unless the number of features is exponential in the input dimension. That analysis crucially relies on the single neuron having a non-zero bias term. In this work we complete the picture by showing that gradient descent can indeed find a near-optimal neuron with non-zero bias. Thus, we see there is indeed essentially a separation between what can be learned using random features and using gradient descent over neural networks.

\begin{remark}[Learning a neuron without bias]
\label{rem:improve no bias}
\cite{yehudai2020learning} studied the problem of learning a single ReLU neuron without bias using gradient descent on a single neuron without bias. For input distributions that are not spherically symmetric they showed that gradient descent with random initialization near zero converges to the global minimum w.p. at least $\frac{1}{2}-o_d(1)$. Their result is also under Assumption~\ref{assum:spread alpha beta}. 
An immediate corollary from 
the discussion above
is that if we learn a single neuron without bias using gradient descent with random initialization on a single neuron with bias, then the algorithm converges to the global minimum w.p. at least $1-e^{\Omega(d)}$. 
Moreover, our proof technique can be easily adapted to the setting of learning a single neuron without bias using gradient descent on a single neuron without bias, namely, the setting studied in \cite{yehudai2020learning}. It can be shown that in this setting gradient descent converges  w.h.p to the global minimum. Thus, our technique allows us to improve the result of \cite{yehudai2020learning} from probability $\frac{1}{2}-o_d(1)$ to probability $1-e^{\Omega(d)}$.
\end{remark}

\section{Convergence for spread and symmetric distributions}
\label{sec: symmetric dist}

In this section we show that under a certain set of assumptions, different from the assumptions in Section~\ref{sec:init better than trivial}, it is possible to show linear convergence of gradient descent to the global minimum. The assumptions we make for this theorem are as follows:

\begin{assumption}\label{assum:main assum convergence}
\;
\begin{enumerate}[leftmargin=*]
    \item The target vector $\bv$ satisfies that $b_\bv \geq 0$ and $\norm{\tilde{\bv}} = 1$.
    \item The distribution $\tilde{\Dcal}$ over the first $d$ coordinates is spherically symmetric.
	\item Assumption \ref{assum:spread alpha beta} holds, and denoting by $\tau := \frac{\mathbb{E}_{\tilde{\bx} \sim \tilde{\Dcal}}\left[|\tilde{x}_1 \tilde{x}_2|\right]}{\mathbb{E}_{\tilde{\bx} \sim \tilde{\Dcal}}\left[\tilde{x}_1^2\right]}$, then $\alpha \geq 2.5\sqrt{2}\cdot \max\left\{1,\frac{1}{\sqrt{\tau}}\right\}$ where $\alpha$ is from Assumption \ref{assum:spread alpha beta}.
    \item Denote by $c:=\mathbb{E}_{\tilde{\bx} \sim \tilde{\Dcal}}\left[\norm{\tilde{\bx}}^4\right]$, then $c < \infty$.
\end{enumerate}
\end{assumption}
We note that item (3) considers $x_1,x_2$, but due to the assumption on the symmetry of the distribution (item (2)), the assumption in item (3) holds for every $x_i, x_j$.
Under these assumptions, we prove the following theorem:
\begin{theorem}\label{thm:symmetric dist}
Assume we initialize $\bw_0$ such that $\norm{{\bw}_0 - {\bv}}^2 < 1$, $b_{\bw_0} \geq 0$ and that Assumption \ref{assum:main assum convergence} 
holds. Then, there is a universal constant $C$, such that using gradient descent on $F(\bw)$ with step size $\eta < C\cdot \frac{\beta}{c\alpha^2}\min\{1,\tau\}$ yields that for every $t$ we have $\norm{\bw_t - \bv}^2 \leq (1-\eta\lambda)^t\norm{\bw_0 - \bv}^2~$, for $\lambda = C\cdot\frac{\beta}{c\alpha^2}$.
\end{theorem}

This result has several advantages and disadvantages compared to those of the previous section. The main disadvantage is that the assumptions are generally more stringent: We focus only on positive target biases ($b_{\bv}\geq 0$) and spherically symmetric distributions $\tilde{\Dcal}$. Also we require a certain technical assumption on the the distribution, as specified by $\tau$, which are satisfied for standard spherically symmetric distributions, but is a bit non-trivial\footnote{For example, for standard Gaussian distribution, we have that $\tau = \frac{2}{\pi} \approx 0.63$, hence we can take $\alpha = 4.5$, and $\beta = O(1)$. Since the distribution $\tilde{\Dcal}$ is symmetric, we present the assumption w.l.o.g with respect to the first $2$ coordinates.}. Finally, the assumption on the initialization ($\norm{\bw_0-\bv}^2<1$ and $b_{\bw_0} \geq 0$) is much more restrictive (although see Remark \ref{remark:randominit} below). 
In contrast, the initialization assumption in the previous section holds with probability close to $1$ with random initialization.  On the positive side, the convergence rate does not depend on the initialization, i.e., here by initializing with any $\bw_0$ such that $\norm{\bw_0-\bv}^2<1$ and $b_{\bw_0} \geq 0$, we get a convergence rate that only depends on the input distribution. On the other hand, in \thmref{thm:smooth linear rate}, the convergence rate depends on the parameter $\delta$ which depends on the initialization. Also, the distribution is not necessarily bounded -- we only require its fourth moment to be bounded. 

\begin{remark}[Random initialization]\label{remark:randominit}
For $b_\bv=0$ the initialization assumption ($\norm{\bw_0-\bv}^2<1)$ is satisfied with probability close to $1/2$ with standard initializations, see Lemma 5.1 from \cite{yehudai2020learning}). For $b_\bv >0$, a similar argument applies if $b_{\bw}$ is initialized close enough to $b_{\bv}$. 
\end{remark}


The proof of the theorem is quite different from the proofs in \secref{sec:init better than trivial}, and is more geometrical in nature, extending previously used techniques from \cite{yehudai2020learning,frei2020agnostic}. It contains two major parts: The first part is an extension of the methods from \cite{yehudai2020learning} to the case of adding a bias term. Specifically, we show a lower bound on $\inner{\nabla F(\bw), \bw-\bv}$, which depends on both the angle between $\tilde{\bw}$ and $\tilde{\bv}$, and the bias terms $b_\bw$ and $b_\bv$ (see \thmref{thm:innerprod}). This result implies that for suitable values of $\bw$, gradient descent will decrease the distance from $\bv$. The second part of the proof is showing that throughout the optimization process, $\bw$ will stay in an area where we can apply the result above. Specifically, the intricate part is showing that the term $-\frac{b_\bw}{\norm{\tilde{\bw}}}$ does not get too large. Note that due to \thmref{thm:characterization}, we know that keeping this term small means that $\bw$ stays away from the cone of bad critical points which are not the global minimum. The full proof can be found in \appref{appen:proof from symmetric dits}.

\section{Discussion}
In this work we studied the problem of learning a single neuron with a bias term using gradient descent. We showed several negative results, indicating that adding a bias term makes the problem more difficult than without a bias term. Next, we gave a characterization of the critical points of the problem under some assumptions on the input distribution, showing that there is a manifold of critical points which are not the global minimum. We proved two convergence results using different techniques and under different assumptions. Finally, we showed that under mild assumptions on the input distribution, reaching the global minimum can be achieved by standard random initialization.



We emphasize that previous works studying the problem of a single neuron either considered non-standard algorithms (e.g. Isotron), or required assumptions on the input distribution which do not allow a bias term. Hence, this is the first work we are aware of which gives positive and negative results on the problem of learning a single neuron with a bias term using gradient methods.

In this work we focused on the gradient descent algorithm. We believe that our results can also be extended to the commonly used SGD algorithm, using similar techniques to \cite{yehudai2020learning,shamir2015stochastic}, and leave it for future work. Another interesting future direction is analyzing other previously studied settings, but with the addition of a bias term. These settings can include convolutional networks, two layers neural networks, and agnostic learning of a single neuron.

\subsection*{Acknowledgements}
This research is supported in part by European Research Council (ERC) grant 754705.


\setcitestyle{numbers}
\bibliographystyle{abbrvnat}
\bibliography{my_bib}

\begin{thebibliography}{23}
\providecommand{\natexlab}[1]{#1}
\providecommand{\url}[1]{\texttt{#1}}
\expandafter\ifx\csname urlstyle\endcsname\relax
  \providecommand{\doi}[1]{doi: #1}\else
  \providecommand{\doi}{doi: \begingroup \urlstyle{rm}\Url}\fi

\bibitem[Brutzkus and Globerson(2017)]{brutzkus2017globally}
A.~Brutzkus and A.~Globerson.
\newblock Globally optimal gradient descent for a convnet with gaussian inputs.
\newblock In \emph{Proceedings of the 34th International Conference on Machine
  Learning-Volume 70}. JMLR. org, 2017.

\bibitem[Bubeck(2014)]{bubeck2014convex}
S.~Bubeck.
\newblock Convex optimization: Algorithms and complexity.
\newblock \emph{arXiv preprint arXiv:1405.4980}, 2014.

\bibitem[Diakonikolas et~al.(2020)Diakonikolas, Goel, Karmalkar, Klivans, and
  Soltanolkotabi]{diakonikolas2020approximation}
I.~Diakonikolas, S.~Goel, S.~Karmalkar, A.~R. Klivans, and M.~Soltanolkotabi.
\newblock Approximation schemes for relu regression.
\newblock In \emph{Conference on Learning Theory}, pages 1452--1485. PMLR,
  2020.

\bibitem[Du et~al.(2017)Du, Lee, and Tian]{du2017convolutional}
S.~S. Du, J.~D. Lee, and Y.~Tian.
\newblock When is a convolutional filter easy to learn?
\newblock \emph{arXiv preprint arXiv:1709.06129}, 2017.

\bibitem[Frei et~al.(2020)Frei, Cao, and Gu]{frei2020agnostic}
S.~Frei, Y.~Cao, and Q.~Gu.
\newblock Agnostic learning of a single neuron with gradient descent.
\newblock \emph{arXiv preprint arXiv:2005.14426}, 2020.

\bibitem[Goel et~al.(2017)Goel, Kanade, Klivans, and Thaler]{goel2017reliably}
S.~Goel, V.~Kanade, A.~Klivans, and J.~Thaler.
\newblock Reliably learning the relu in polynomial time.
\newblock In \emph{Conference on Learning Theory}, pages 1004--1042. PMLR,
  2017.

\bibitem[Goel et~al.(2019)Goel, Karmalkar, and Klivans]{goel2019time}
S.~Goel, S.~Karmalkar, and A.~Klivans.
\newblock Time/accuracy tradeoffs for learning a relu with respect to gaussian
  marginals.
\newblock \emph{arXiv preprint arXiv:1911.01462}, 2019.

\bibitem[Goel et~al.(2020)Goel, Klivans, Manurangsi, and
  Reichman]{goel2020tight}
S.~Goel, A.~Klivans, P.~Manurangsi, and D.~Reichman.
\newblock Tight hardness results for training depth-2 relu networks.
\newblock \emph{arXiv preprint arXiv:2011.13550}, 2020.

\bibitem[Kakade et~al.(2011)Kakade, Kanade, Shamir, and
  Kalai]{kakade2011efficient}
S.~M. Kakade, V.~Kanade, O.~Shamir, and A.~Kalai.
\newblock Efficient learning of generalized linear and single index models with
  isotonic regression.
\newblock In \emph{Advances in Neural Information Processing Systems}, pages
  927--935, 2011.

\bibitem[Kalai and Sastry(2009)]{kalai2009isotron}
A.~T. Kalai and R.~Sastry.
\newblock The isotron algorithm: High-dimensional isotonic regression.
\newblock In \emph{COLT}. Citeseer, 2009.

\bibitem[Kalan et~al.(2019)Kalan, Soltanolkotabi, and
  Avestimehr]{kalan2019fitting}
S.~M.~M. Kalan, M.~Soltanolkotabi, and A.~S. Avestimehr.
\newblock Fitting relus via sgd and quantized sgd.
\newblock In \emph{2019 IEEE International Symposium on Information Theory
  (ISIT)}, pages 2469--2473. IEEE, 2019.

\bibitem[Kamath et~al.(2020)Kamath, Montasser, and
  Srebro]{kamath2020approximate}
P.~Kamath, O.~Montasser, and N.~Srebro.
\newblock Approximate is good enough: Probabilistic variants of dimensional and
  margin complexity.
\newblock In \emph{Conference on Learning Theory}, pages 2236--2262. PMLR,
  2020.

\bibitem[Malach et~al.(2021)Malach, Kamath, Abbe, and
  Srebro]{malach2021quantifying}
E.~Malach, P.~Kamath, E.~Abbe, and N.~Srebro.
\newblock Quantifying the benefit of using differentiable learning over tangent
  kernels.
\newblock \emph{arXiv preprint arXiv:2103.01210}, 2021.

\bibitem[Mei et~al.(2016)Mei, Bai, and Montanari]{mei2016landscape}
S.~Mei, Y.~Bai, and A.~Montanari.
\newblock The landscape of empirical risk for non-convex losses.
\newblock \emph{arXiv preprint arXiv:1607.06534}, 2016.

\bibitem[Oymak and Soltanolkotabi(2018)]{oymak2018overparameterized}
S.~Oymak and M.~Soltanolkotabi.
\newblock Overparameterized nonlinear learning: Gradient descent takes the
  shortest path?
\newblock \emph{arXiv preprint arXiv:1812.10004}, 2018.

\bibitem[Paszke et~al.(2019)Paszke, Gross, Massa, Lerer, Bradbury, Chanan,
  Killeen, Lin, Gimelshein, Antiga, et~al.]{paszke2019pytorch}
A.~Paszke, S.~Gross, F.~Massa, A.~Lerer, J.~Bradbury, G.~Chanan, T.~Killeen,
  Z.~Lin, N.~Gimelshein, L.~Antiga, et~al.
\newblock Pytorch: An imperative style, high-performance deep learning library.
\newblock \emph{arXiv preprint arXiv:1912.01703}, 2019.

\bibitem[Shamir(2015)]{shamir2015stochastic}
O.~Shamir.
\newblock A stochastic pca and svd algorithm with an exponential convergence
  rate.
\newblock In \emph{International Conference on Machine Learning}, pages
  144--152, 2015.

\bibitem[Soltanolkotabi(2017)]{soltanolkotabi2017learning}
M.~Soltanolkotabi.
\newblock Learning relus via gradient descent.
\newblock In \emph{Advances in Neural Information Processing Systems}, pages
  2007--2017, 2017.

\bibitem[Tan and Vershynin(2019)]{tan2019online}
Y.~S. Tan and R.~Vershynin.
\newblock Online stochastic gradient descent with arbitrary initialization
  solves non-smooth, non-convex phase retrieval.
\newblock \emph{arXiv preprint arXiv:1910.12837}, 2019.

\bibitem[Tian(2017)]{tian2017analytical}
Y.~Tian.
\newblock An analytical formula of population gradient for two-layered relu
  network and its applications in convergence and critical point analysis.
\newblock In \emph{Proceedings of the 34th International Conference on Machine
  Learning-Volume 70}, pages 3404--3413. JMLR. org, 2017.

\bibitem[Vardi and Shamir(2020)]{vardi2020implicit}
G.~Vardi and O.~Shamir.
\newblock Implicit regularization in relu networks with the square loss.
\newblock \emph{arXiv preprint arXiv:2012.05156}, 2020.

\bibitem[Yehudai and Shamir(2019)]{yehudai2019power}
G.~Yehudai and O.~Shamir.
\newblock On the power and limitations of random features for understanding
  neural networks.
\newblock \emph{arXiv preprint arXiv:1904.00687}, 2019.

\bibitem[Yehudai and Shamir(2020)]{yehudai2020learning}
G.~Yehudai and O.~Shamir.
\newblock Learning a single neuron with gradient methods.
\newblock \emph{arXiv preprint arXiv:2001.05205}, 2020.

\end{thebibliography}


\appendix
\section*{Appendices}

\section{Auxiliary Results}
In this appendix we extend several key results from \cite{yehudai2020learning} for the case of adding a bias term. Specifically, we extend Theorem 4.2 from \cite{yehudai2020learning} which shows that under mild assumptions on the distribution, the gradient of the loss points in a good direction which depends on the angle between the learned vector $\bw$ and the target $\bv$. We also bound the volume of a certain set in $\reals^2$, which can be seen as an extension of Lemma B.1 from \cite{yehudai2020learning}.

\begin{lemma}\label{lem:vol of P}
Let $P=\{\by\in\reals^2: {{\bw}}^\top \by > b, {{\bv}}^\top \by > b, \norm{\by}\leq \alpha\}$ for $b \in \reals$ and ${\bw},{\bv}\in\reals^2$ with $\norm{\bw},\norm{\bv} =1$ and $\theta({\bw},{\bv}) \leq \pi - \delta$ for $\delta\in [0,\pi]$. If $b<\alpha\sin\left(\frac{\delta}{2}\right)$ then $\text{Vol}(P) \geq \frac{\left(\alpha\sin\left(\frac{\delta}{2}\right) - b\right)^2}{4\sin\left(\frac{\delta}{2}\right)}$.
\end{lemma}

\begin{proof}
The volume of $P$ is smallest when the angle is exactly $\pi - \delta$, thus we can lower bound the volume by assuming that $\theta({\bw},{\bv}) = \pi -\delta$. Next, we can rotate to coordinates to consider without loss of generality the volume of the set

\[
P'=\left\{(y_1,y_2)\in\reals^2:\theta((y_1,y_2-b'),\be_2) \leq \delta/2, \norm{(y_1,y_2)}\leq \alpha\right\}~,
\]

where $b' = \frac{b}{\sin(\delta/2)}$ and $\be_2=(0,1)$. Let $P'' = \{(x,y)\in\reals^2: x^2 + (y-b')^2 \leq (\alpha-b')^2\}$ be the disc of radius $\alpha-b'$ around the point $(0,b')$. It is enough to bound the volume of $P'\cap P''$. We define the rectangular sets:

\begin{align*}
    P_1 &= \left[\frac{(\alpha - b')}{2}\sin\left(\frac{\delta}{4}\right),{(\alpha - b')}\sin\left(\frac{\delta}{4}\right)\right] \times \left[b' + \frac{(\alpha - b')}{2}\cos\left(\frac{\delta}{4}\right),b' + (\alpha - b')\cos\left(\frac{\delta}{4}\right)\right] \\
    P_2 &= \left[-{(\alpha - b')}\sin\left(\frac{\delta}{4}\right),-\frac{(\alpha - b')}{2}\sin\left(\frac{\delta}{4}\right)\right] \times \left[b' + \frac{(\alpha - b')}{2}\cos\left(\frac{\delta}{4}\right),b' + (\alpha - b')\cos\left(\frac{\delta}{4}\right)\right]
\end{align*}

See \figref{fig:P'} for an illustration. We have that $P_1,P_2\subseteq P'\cap P''$. We will show it for $P_1$, the same argument also works for $P_2$. First, $P_1\subseteq P''$ is immediate by the definition of the two sets. For $P'$, the straight line in the boundary of $P'$ is defined by $y_2 = b' + y_1\cdot \frac{\cos\left(\frac{\delta}{2}\right)}{\sin\left(\frac{\delta}{2}\right)}$. It can be seen that each vertex of the rectangle $P_1$, is above this line. Moreover, the norm of each vertex of $P_1$ is at most $\alpha$. Hence all the vertices are inside $P'$, which means that $P_1\subseteq P'$.  In  total we get:
\begin{align*}
    \text{Vol}(P) \geq \text{Vol}(P'\cap P'') &\geq \text{Vol}(P_1\cup P_2)\\
    & = \frac{(\alpha - b')^2}{2}\sin\left(\frac{\delta}{4}\right)\cos\left(\frac{\delta}{4}\right) \\
    &= \frac{\left(\alpha\sin\left(\frac{\delta}{2}\right) - b\right)^2}{4\sin\left(\frac{\delta}{2}\right)}
\end{align*}
\end{proof}

\begin{figure}
    \centering
    \includegraphics[width=3.5in]{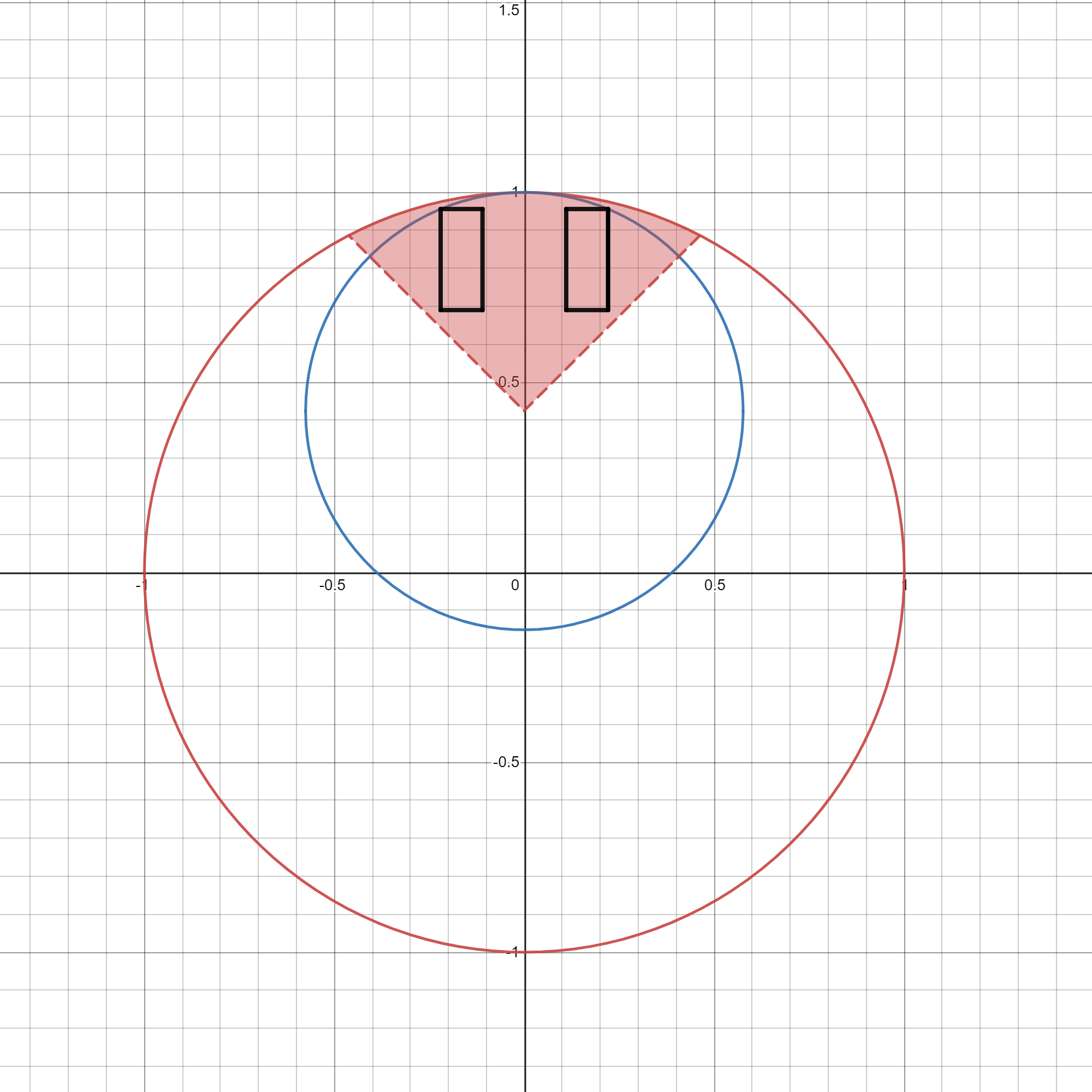}    \caption{An illustration of the set $P'$ (in red), the circle $P''$ (in blue) and the two rectangles $P_1,P_2$ (in black), for the case of $\delta = \pi/2$, $\alpha=1$ and $b=0.3$. For $b=0$, $P'$ would be a pie slice, and the blue circle $P''$ will coincide with the red circle.}
    \label{fig:P'}
\end{figure}

\begin{theorem}\label{thm:innerprod}
Let $\bw,\bv\in\reals^{d+1}$ , denote by $\tilde{\bw},\tilde{\bv}$ their first $d$ coordinates and by $b_\bw,b_\bv$ their last coordinate. Assume that $\theta(\tilde{\bw},\tilde{\bv}) \leq \pi - \delta$ for some $\delta\in[0,\pi)$, and that the distribution $\cd$ is such that its first $d$ coordinates satisfy Assumption 4.1 (1) from \cite{yehudai2020learning}, and that its last coordinate is a constant $1$. Denote $b' =  \max\{-b_\bw/\norm{\tilde{\bw}},-b_\bv/\norm{\tilde{\bv}},0\} \cdot \frac{1}{\sin\left(\frac{\delta}{2}\right)}$, and assume that  $b'<\alpha$, then:
\[
 \inner{\nabla F(\bw),\bw-\bv} \geq \frac{(\alpha-b')^4\sin\left(\frac{\delta}{4}\right)^3\beta}{8^4}\cdot\min\left\{ 1, \frac{1}{\alpha^2}\right\}\norm{\bw-\bv}^2
\]
\end{theorem}


\begin{proof}
Let $\tilde{\bx}$ be the first $d$ coordinates of $\bx$. We have that:
\begin{align}
    \inner{\nabla F(\bw),\bw-\bv} &= \mathbb{E}_{\bx\sim \cd}\left[\sigma'(\bw^\top \bx)(\sigma(\bw^\top \bx) - \sigma(\bv^\top \bx))(\bw^\top \bx - \bv^\top \bx) \right] \nonumber\\
    & \geq \mathbb{E}_{\bx\sim \cd} \left[ \mathbbm{1}(\bw^\top \bx >0, \bv^\top \bx >0) (\bw^\top \bx - \bv^\top \bx)^2 \right] \nonumber\\
    & =  \norm{\bw-\bv}^2 \cdot \mathbb{E}_{\bx\sim \cd} \left[ \mathbbm{1}(\tilde{\bw}^\top \tilde{\bx} >-b_\bw, \tilde{\bv}^\top \tilde{\bx} >-b_\bv) ((\overline{\bw-\bv})^\top\bx)^2 \right] \nonumber\\
    & \geq  \norm{\bw-\bv}^2 \cdot\inf_{\bu\in\text{span}\{\bw,\bv\},\norm{\bu}=1} \mathbb{E}_{\bx\sim \cd} \left[ \mathbbm{1}(\tilde{\bw}^\top \tilde{\bx} >-b_\bw, \tilde{\bv}^\top \tilde{\bx} >-b_\bv) (\bu^\top\bx)^2 \right] \nonumber
\end{align}

Let $b = \max\{-b_\bw/\norm{\tilde{\bw}},-b_\bv/\norm{\tilde{\bv}},0\}$, then we can bound the above equation by:
\begin{align}
    & \norm{\bw-\bv}^2 \cdot\inf_{\bu\in\text{span}\{\bw,\bv\},\norm{\bu}=1} \mathbb{E}_{\bx\sim \cd} \left[ \mathbbm{1}(\overline{\tilde{\bw}}^\top \tilde{\bx} >b, \overline{\tilde{\bv}}^\top \tilde{\bx} >b) (\bu^\top\bx)^2 \right] \nonumber\\ 
    & \geq \norm{\bw-\bv}^2 \cdot\inf_{\bu\in\text{span}\{\bw,\bv\},\norm{\bu}=1} \mathbb{E}_{\tilde{\bx}\sim \tilde{\cd}} \left[ \mathbbm{1}(\overline{\tilde{\bw}}^\top \tilde{\bx} >b, \overline{\tilde{\bv}}^\top \tilde{\bx} >b, \norm{\tilde{\bx}} \leq \alpha) (\tilde{\bu}^\top\tilde{\bx} + b_{\bu})^2 \right]\label{eq:inner prod bound}
\end{align}

Here $b_{\bu}$ is the bias term of $\bu$, $\tilde{\bu}$ are the first $d$ coordinates of $\bu$ and $\tilde{\cd}$ is the marginal distribution of $\bx$ on its first $d$ coordinates. Note that since the last coordinate represents the bias term, then the distribution on the last coordinate of $\bx$ is a constant $1$. The condition that $\norm{\bu}=1$ (equivalently $\norm{\bu}^2=1$) translates to $\norm{\tilde{\bu}}^2 + b_{\bu}^2 = 1$.

Our goal is to bound the term inside the infimum. Note that the expression inside the distribution depends just on inner products of $\tilde{\bx}$ with $\tilde{\bw}$ or $\tilde{\bv}$, hence we can consider the marginal distribution $\cd_{\tilde{\bw},\tilde{\bv}}$ of $\tilde{\bx}$ on the 2-dimensional subspace spanned by $\tilde{\bw}$ and $\tilde{\bv}$ (with density function $p_{\tilde{\bw},\tilde{\bv}}$). Let $\hat{\bw}$ and $\hat{\bv}$ be the projections of $\tilde{\bw}$ and $\tilde{\bv}$ on that subspace. Let $P=\{\by\in\reals^2: \overline{\hat{\bw}}^\top \by > b, \overline{\hat{\bv}}^\top \by > b, \norm{\by}\leq \alpha\}$, then we can bound \eqref{eq:inner prod bound} with:

\begin{align*}
    & \norm{\bw-\bv}^2 \cdot \inf_{\bu\in \reals^2, b_{\bu}\in\reals: \norm{\bu}^2 + b_{\bu}^2 = 1} \mathbb{E}_{\by\sim \cd_{\tilde{\bw},\tilde{\bv}}}\left[\mathbbm{1}(\by\in P)\cdot(\bu^\top\by + b_{\bu})^2\right] \\
    & = \norm{\bw-\bv}^2 \cdot \inf_{\bu\in \reals^2, b_{\bu}\in\reals: \norm{\bu}^2 + b_{\bu}^2 = 1} \int_{\by\in\reals^2}\mathbbm{1}(\by\in P)\cdot(\bu^\top\by + b_{\bu})^2 p_{\tilde{\bw},\tilde{\bv}}(\by)d\by \\
    & \geq \beta\norm{\bw-\bv}^2 \cdot \inf_{\bu\in \reals^2, b_{\bu}\in\reals: \norm{\bu}^2 + b_{\bu}^2 = 1} \int_{\by\in P}(\bu^\top\by + b_{\bu})^2 d\by
\end{align*}

Combining with \propref{prop:integral bound} finishes the proof
\end{proof}

\begin{proposition}\label{prop:integral bound}
Let $P=\{\by\in\reals^2: \overline{\hat{\bw}}^\top \by > b, \overline{\hat{\bv}}^\top \by > b, \norm{\by}\leq \alpha\}$ for $b \in \reals$ and $\hat{\bw},\hat{\bv}\in\reals^2$ with $\theta(\hat{\bw},\hat{\bv}) \leq \pi - \delta$ for $\delta\in [0,\pi]$. Then
\[
\inf_{\bu\in \reals^2, b_{\bu}\in\reals: \norm{\bu}^2 + b_{\bu}^2 = 1} \int_{\by\in P}(\bu^\top\by + b_{\bu})^2 d\by \geq \frac{(\alpha-b')^4\sin\left(\frac{\delta}{4}\right)^3}{8^4}\cdot\min\left\{ 1, \frac{1}{\alpha^2}\right\}
\]
for $b' = \frac{b}{\sin\left(\frac{\delta}{2}\right)}$.
\end{proposition}

\begin{proof}
As in the proof of \lemref{lem:vol of P}, we consider the rectangular sets:
\begin{align*}
    P_1 &= \left[\frac{(\alpha - b')}{2}\sin\left(\frac{\delta}{4}\right),{(\alpha - b')}\sin\left(\frac{\delta}{4}\right)\right] \times \left[b' + \frac{(\alpha - b')}{2}\cos\left(\frac{\delta}{4}\right),b' + (\alpha - b')\cos\left(\frac{\delta}{4}\right)\right] \\
    P_2 &= \left[-{(\alpha - b')}\sin\left(\frac{\delta}{4}\right),-\frac{(\alpha - b')}{2}\sin\left(\frac{\delta}{4}\right)\right] \times \left[b' + \frac{(\alpha - b')}{2}\cos\left(\frac{\delta}{4}\right),b' + (\alpha - b')\cos\left(\frac{\delta}{4}\right)\right]
\end{align*}
with $b' = \frac{b}{\sin(\delta/2)}$. Since we have $P_1\cup P_2\subseteq P$, and the function inside the integral is positive, we can lower bound the target integral by integrating only over $P_1\cup P_2$. Now we have: 

\begin{align*}
    & \inf_{\bu\in \reals^2, b_{\bu}\in\reals: \norm{\bu}^2 + b_{\bu}^2 = 1} \int_{\by\in P}(\bu^\top\by + b_{\bu})^2 d\by  \\
    & \geq \inf_{u_1,u_2,b_{\bu}\in\reals: u_1^2 + u_2^2 + b_{\bu}^2 = 1} \int_{\by\in P_1\cup P_2}(u_1y_1+u_2y_2 + b_{\bu})^2 d\by \\
    & = \inf_{u_1,u_2,b_{\bu}\in\reals: u_1^2 + u_2^2 + b_{\bu}^2 = 1} \int_{\by\in P_1\cup P_2}(u_1y_1)^2d\by + \int_{\by\in P_1\cup P_2}(u_2y_2+b_{\bu})^2d\by + \int_{\by\in P_1\cup P_2}2u_1y_1(u_2y_2+b_{\bu})d\by   \\
    & = \inf_{u_1,u_2,b_{\bu}\in\reals: u_1^2 + u_2^2 + b_{\bu}^2 = 1} \int_{\by\in P_1\cup P_2}(u_1y_1)^2d\by + \int_{\by\in P_1\cup P_2}(u_2y_2+b_{\bu})^2d\by
\end{align*}
where in the last equality we used that $P_1\cup P_2$ are symmetric around the $y_2$ axis, i.e. $(y_1,y_2)\in P_1\cup P_2$ iff $(-y_1,y_2)\in P_1\cup P_2$. By the condition that $u_1^2+u_2^2+b_{\bu}^2=1$ we know that either $u_1^2\geq \frac{1}{2}$ or $u_2^2 + b_\bu^2  \geq \frac{1}{2}$. Using that both integrals above are positive, we can lower bound:
\begin{align*}
    &\inf_{u_1,u_2,b_{\bu}\in\reals: u_1^2 + u_2^2 + b_{\bu}^2 = 1} \int_{\by\in P_1\cup P_2}(u_1y_1)^2d\by + \int_{\by\in P_1\cup P_2}(u_2y_2+b_{\bu})^2d\by \\
    \geq & \min\left\{\frac{1}{2}\int_{\by\in P_1\cup P_2}y_1^2d\by, \inf_{u_2,u_3\in\reals:  u_2^2 + u_3^2 = \frac{1}{2}} \int_{\by\in P_1\cup P_2} (u_2y_2+u_3)^2d\by \right\}~.
\end{align*}

We will now lower bound both terms in the above equation. For the first term, note that for every $\by\in P_1\cup P_2$ we have that $|y_1| \geq \frac{(\alpha -b')}{2}\sin\left(\frac{\delta}{4}\right)$. Hence we have:
\begin{align}\label{eq:int bound on y 1}
    &\frac{1}{2}\int_{\by\in P_1\cup P_2}y_1^2d\by \geq \nonumber \\
     \geq & \frac{1}{2}\int_{\by\in P_1\cup P_2}\frac{(\alpha -b')^2}{4}\sin\left(\frac{\delta}{4}\right)^2d\by \nonumber \\
    = & \frac{(\alpha -b')^2}{8}\sin\left(\frac{\delta}{4}\right)^2 \cdot \frac{(\alpha - b')^2}{2}\sin\left(\frac{\delta}{4}\right)\cos\left(\frac{\delta}{4}\right) \nonumber \\
    \geq & \frac{(\alpha - b')^4}{16\sqrt{2}}\sin\left(\frac{\delta}{4}\right)^3
\end{align}
where in the last inequality we used that $\delta\in[0,\pi]$, hence $\delta/4\in [0,\pi/4]$.

For the second term we have:
\begin{align}
    &\inf_{u_2,u_3\in\reals:  u_2^2 + u_3^2 = \frac{1}{2}} \int_{\by\in P_1\cup P_2} (u_2y_2+u_3)^2d\by \nonumber\\
     =& \inf_{u\in \left[-\frac{1}{\sqrt{2}},\frac{1}{\sqrt{2}}\right]} \int_{\by\in P_1\cup P_2} \left(uy_2+\sqrt{\frac{1}{2}-u^2}\right)^2d\by \nonumber\\
     =& {(\alpha - b')}\sin\left(\frac{\delta}{4}\right)\inf_{u\in \left[-\frac{1}{\sqrt{2}},\frac{1}{\sqrt{2}}\right]} \int_{y_2\in C} \left(uy_2+\sqrt{\frac{1}{2}-u^2}\right)^2dy_2~.\label{eq:bound on inf integral}
\end{align}
The last equality is given by changing the order of integration into integral over $y_2$ and then over $y_1$, denoting the interval $C = \left[b' + \frac{(\alpha - b')}{2}\cos\left(\frac{\delta}{4}\right),b' + (\alpha - b')\cos\left(\frac{\delta}{4}\right)\right]$, and noting that the term inside the integral does not depend on $y_1$.

Fix some $u\in \left[-\frac{1}{\sqrt{2}},\frac{1}{\sqrt{2}}\right]$. If $u=0$, then we can bound \eqref{eq:bound on inf integral} by $\frac{(\alpha-b')^2}{4}\sin\left(\frac{\delta}{4}\right)\cos\left(\frac{\delta}{4}\right)$. Assume $u\neq 0$, we split into cases and bound the term inside the integral:

\textbf{Case I:}  $\left|\frac{\sqrt{\frac{1}{2}-u^2}}{u}\right| \geq b'+\frac{3}{4}\cdot(\alpha-b')\cos\left(\frac{\delta}{4}\right)$. In this case, solving the inequality for $u$ we have $|u| \leq \sqrt{\frac{1}{2+2\left(b'+\frac{3}{4}(\alpha - b')\cos\left(\frac{\delta}{4}\right)\right)^2}}$. Hence, we can also bound:
\[
\sqrt{\frac{1}{2}-u^2} \geq \sqrt{\frac{1}{2} - \frac{1}{2+2\left(b'+\frac{3}{4}(\alpha - b')\cos\left(\frac{\delta}{4}\right)\right)^2}} = \sqrt{\frac{\left(b'+\frac{3}{4}(\alpha - b')\cos\left(\frac{\delta}{4}\right)\right)^2}{2+2\left(b'+\frac{3}{4}(\alpha - b')\cos\left(\frac{\delta}{4}\right)\right)^2}}
\]
In particular, for every $y_2\in \left[b' + \frac{(\alpha - b')}{2}\cos\left(\frac{\delta}{4}\right),b' + \frac{5(\alpha - b')}{8}\cos\left(\frac{\delta}{4}\right)\right]$ we get that:
\begin{align*}
    &\left|u y_2 + \sqrt{\frac{1}{2} - u^2}\right| \\
    \geq & \left| \sqrt{\frac{\left(b'+\frac{3}{4}(\alpha - b')\cos\left(\frac{\delta}{4}\right)\right)^2}{2+2\left(b'+\frac{3}{4}(\alpha - b')\cos\left(\frac{\delta}{4}\right)\right)^2}} - \sqrt{\frac{\left(b'+\frac{5}{8}(\alpha - b')\cos\left(\frac{\delta}{4}\right)\right)^2}{2+2\left(b'+\frac{3}{4}(\alpha - b')\cos\left(\frac{\delta}{4}\right)\right)^2}}\right| \\
    \geq & \frac{(\alpha - b')\cos\left(\frac{\delta}{4}\right)}{8\sqrt{2+2\left(b'+(\alpha-b')\cos\left(\frac{\delta}{4}\right)  \right)^2}}
\end{align*}

\textbf{Case II:}  $\left|\frac{\sqrt{\frac{1}{2}-u^2}}{u}\right| < b'+\frac{3}{4}\cdot(\alpha-b')\cos\left(\frac{\delta}{4}\right)$. Using the same reasoning as above, we get for every $y_2\in \left[b' + \frac{7(\alpha - b')}{8}\cos\left(\frac{\delta}{4}\right),b' + (\alpha - b')\cos\left(\frac{\delta}{4}\right)\right]$ that: 
\[
\left|u y_2 + \sqrt{\frac{1}{2} - u^2}\right| \geq \frac{(\alpha - b')\cos\left(\frac{\delta}{4}\right)}{8\sqrt{2+2\left(b'+(\alpha-b')\cos\left(\frac{\delta}{4}\right)  \right)^2}}
\]

Combining the above cases with \eqref{eq:bound on inf integral} we get that:

\begin{align}\label{eq:int bound on y 2 u 2}
    &\inf_{u_2,u_3\in\reals:  u_2^2 + u_3^2 = \frac{1}{2}} \int_{\by\in P_1\cup P_2} (u_2y_2+u_3)^2d\by  \nonumber\\
    & \geq (\alpha - b')\sin\left(\frac{\delta}{4}\right)\int_{y_2\in C}  \frac{(\alpha - b')^2\cos\left(\frac{\delta}{4}\right)^2}{8^2(2+2\left(b'+(\alpha-b')\cos\left(\frac{\delta}{4}\right)  \right)^2)}dy_2 \nonumber\\
    & \geq \frac{(\alpha-b')^4\cos\left(\frac{\delta}{4}\right)^3\sin\left(\frac{\delta}{4}\right)}{2\cdot 8^2\left(2+2\left(b'+(\alpha-b')\cos\left(\frac{\delta}{4}\right)\right)\right)^2} \nonumber \\
    & \geq \frac{(\alpha-b')^4\sin\left(\frac{\delta}{4}\right)^3}{2\cdot 8^2\sqrt{2}\left(2+2\left(b'+(\alpha-b')\cos\left(\frac{\delta}{4}\right)\right)\right)^2} \nonumber \\
    &\geq \frac{(\alpha-b')^4\sin\left(\frac{\delta}{4}\right)^3}{8^4}\cdot\min\left\{ 1, \frac{1}{\alpha^2}\right\}
\end{align}
where in the second inequality we used that for $\delta\in [0,\pi]$ we have $\sin\left(\frac{\delta}{4}\right)\leq \cos\left(\frac{\delta}{4}\right)$, and in the last inequality we used that $b'\leq \alpha$, and $(\alpha-b')\cos\left(\frac{\delta}{4}\right)\leq \alpha$. Combining \eqref{eq:int bound on y 1} with \eqref{eq:int bound on y 2 u 2} finishes the proof.

\end{proof}

\section{Proofs from \secref{sec:negative results}}\label{appen:proofs from neg results}

\subsection{Proof of \thmref{thm:negative example 1}}
\label{appen:proofs from neg result1}

Let $\epsilon > 0$, for the input distribution, we consider the uniform distribution on the ball of radius $\epsilon$. 
Let $b_\bw$ be the last coordinate of $\bw$, and denote by $\tilde{\bw},\tilde{\bx}$ the first $d$ coordinates of $\bw$ and $\bx$. Using the assumption on the initialization of $\bw_0$ and on the boundness of the distribution $\tilde{\Dcal}$ we have:
\[
|\inner{\tilde{\bw}_0,\tilde{\bx}}| \leq \norm{\tilde{\bw}_0}\norm{\tilde{\bx}} \leq \epsilon\sqrt{d}.
\]
Since $b_{\bw_0}$ is also initialized with $U([-1,1])$, w.p $> 1/2 - \epsilon\sqrt{d}$ we have that $b_{\bw_0} <-\epsilon\sqrt{d}$. If this event happens, since the activation is ReLU we get that $\sigma'(\inner{\bw_0,\bx}) = \mathbbm{1}(\inner{\tilde{\bw}_0,\tilde{\bx}} + b_{\bw_0} > 0) = 0$ for every $\tilde{\bx}$ in the support of the distribution. Using \eqref{eq:gradient of a single neuron} we get that $\nabla F(\bw_0) =0$, hence gradient flow will get stuck at its initial value.

\subsection{Proof of \thmref{thm:negative2}}
\label{appen:proofs from neg result2}

\begin{lemma}
\label{lem:conditions for init}
	Let $\bw \in \reals^{d+1}$ such that $b_\bw = 0$, $\tilde{w}_1 < - \frac{4}{\sqrt{d}}$, and $\norm{\tilde{\bw}_{2:d}} \leq 2\sqrt{d}$. Then, 
	\[
		\Pr_{\bx \sim \cd} \left[ \bw^\top \bx \geq 0,  \bv^\top \bx \geq 0\right] = 0~.
	\]
\end{lemma}
\begin{proof}
	If $\bv^\top \bx \geq 0$ then 
	$x_1 \geq r - \frac{r}{2d^2}$
	and hence 
	$x_1^2 \geq r^2 - \frac{r^2}{d^2}$.
	Since we also have 
	$\norm{\tilde{\bx}} \leq r$
	then  
	\[
		\norm{\tilde{\bx}_{2:d}}^2 = \norm{\tilde{\bx}}^2 - x_1^2 \leq r^2 - \left(r^2 - \frac{r^2}{d^2}\right) =  \frac{r^2}{d^2}~.
	\]
	Hence,
	\[
		\Pr_{\bx \sim \cd} \left[ \bw^\top \bx \geq 0,  \bv^\top \bx \geq 0\right] 
		\leq \Pr_{\bx \sim \cd} \left[ \bw^\top \bx \geq 0,x_1 \geq r - \frac{r}{2d^2}, \norm{\tilde{\bx}_{2:d}} \leq \frac{r}{d} \right] ~.
	\]

	Since $b_\bw = 0$, $\norm{\tilde{\bw}_{2:d}} \leq 2\sqrt{d}$ and $\tilde{w}_1 < - \frac{4}{\sqrt{d}}$, then for every $\tilde{\bx} \in \Bcal$  such that 
	$x_1 \geq r - \frac{r}{2d^2} \geq \frac{r}{2}$ and $\norm{\tilde{\bx}_{2:d}} \leq \frac{r}{d}$
	we have
	\[
		\bw^\top \bx
		= \tilde{\bw}^\top \tilde{\bx}
		= \tilde{w}_1 \tilde{x}_1 + \inner{\tilde{\bw}_{2:d}, \tilde{\bx}_{2:d}}
		< - \frac{4}{\sqrt{d}} \cdot \frac{r}{2} + 2\sqrt{d} \cdot \frac{r}{d}
		= 0~.
	\]
	Therefore, $\Pr_{\bx \sim \cd} \left[ \bw^\top \bx \geq 0,  \bv^\top \bx \geq 0\right] = 0$.
\end{proof}

\begin{lemma}
\label{lem:negative2 random init}
	With probability $\frac{1}{2} - o_d(1)$ over the choice of $\bw_0$, we have 
	\[
		\Pr_{\bx \sim \cd} \left[ \bw_0^\top \bx \geq 0,  \bv^\top \bx \geq 0\right] = 0~.
	\]
\end{lemma}
\begin{proof}
	Let $\bw \in \reals^{d+1}$ such that $b_\bw = 0$ and  $\tilde{\bw} \sim \Ncal(\zero,I_{d})$.
	Since $\tilde{w}_1$ has a standard normal distribution, then we have $\tilde{w}_1 < - \frac{4}{\sqrt{d}}$ with probability $\frac{1}{2} - o_d(1)$. Moreover, note that $\norm{\tilde{\bw}_{2:d}}^2$ has a chi-square distribution and the probability of $\norm{\tilde{\bw}_{2:d}}^2 \leq 4d$ is $1-o_d(1)$.
	Hence, by Lemma~\ref{lem:conditions for init}, with probability $\frac{1}{2} - o_d(1)$ over the choice of $\bw$, we have
	\[
		\Pr_{\bx \sim \cd} \left[ \bw^\top \bx \geq 0,  \bv^\top \bx \geq 0\right] = 0~.
	\]
	Therefore,
	\begin{align*}	
		\Pr_{\bx \sim \cd} \left[\rho \frac{\bw^\top}{\norm{\bw}} \bx \geq 0,  \bv^\top \bx \geq 0\right] 
		= \Pr_{\bx \sim \cd} \left[\bw^\top \bx \geq 0,  \bv^\top \bx \geq 0\right] 
		= 0~. 
	\end{align*}
	Since $\rho \frac{\bw^\top}{\norm{\bw}}$ has the distribution of $\bw_0$, the lemma follows. 
\end{proof}


\begin{lemma}
\label{lem:negative2 dynamics}
	Assume that $\bw_0$ satisfies $\Pr_{\bx \sim \cd} \left[\bw_0^\top \bx \geq 0, \bv^\top \bx \geq 0 \right] = 0$.
	Let $\gamma > 0$ and let $\bw \in \reals^{d+1}$ such that $\tilde{\bw} = \gamma \tilde{\bw}_0$, and $b_\bw \leq 0$.
	Then,  $\Pr_{\bx \sim \cd} \left[\bw^\top \bx \geq 0, \bv^\top \bx \geq 0 \right] = 0$.
	Moreover, we have
	\begin{itemize}
		\item	If $-\frac{b_\bw}{\norm{\tilde{\bw}}} < r$, then $\frac{d \tilde{\bw}}{dt} = -s \tilde{\bw}$ for some $s > 0$, and $\frac{d b_\bw}{dt} < 0$.  
		\item If $-\frac{b_\bw}{\norm{\tilde{\bw}}} \geq r$, then $\frac{d \tilde{\bw}}{dt} = \zero$ and $\frac{d b_\bw}{dt} = 0$.
	\end{itemize}
\end{lemma}
\begin{proof}
	For every $\bx$ we have: If $\bw^\top \bx = \gamma \tilde{\bw}_0^\top \tilde{\bx} + b_\bw \geq 0$ then $ \gamma \tilde{\bw}_0^\top \tilde{\bx} \geq 0$, and therefore $\bw_0^\top \bx = \tilde{\bw}_0^\top \tilde{\bx} \geq 0$. Thus 
	\begin{equation}
	\label{eq:prob at t}
		\Pr_{\bx \sim \cd} \left[\bw^\top \bx \geq 0, \bv^\top \bx \geq 0 \right] 
		\leq \Pr_{\bx \sim \cd} \left[\bw_0^\top \bx \geq 0, \bv^\top \bx \geq 0 \right] 
		= 0~.
	\end{equation}

	We have
	\begin{align*}
		- \frac{d \tilde{\bw}}{dt}  
		&= \nabla_{\tilde{\bw}} F(\bw)
		= \E_\bx \left(\sigma(\bw^\top \bx) - \sigma(\bv^\top \bx)\right) \sigma'(\bw^\top \bx) \tilde{\bx}
		\\
		&= \E_\bx \left(\sigma(\bw^\top \bx) - \sigma(\bv^\top \bx)\right) \onefunc(\bw^\top \bx \geq 0) \tilde{\bx}
		\\
		&= \E_\bx \left(\sigma(\bw^\top \bx) - \sigma(\bv^\top \bx)\right) \onefunc(\bw^\top \bx \geq 0,  \bv^\top \bx <0) \tilde{\bx} 
		\\
		&\;\;\;\;\;+ \E_\bx \left(\sigma(\bw^\top \bx) - \sigma(\bv^\top \bx)\right) \onefunc(\bw^\top \bx \geq 0,  \bv^\top \bx \geq 0) \tilde{\bx}
		\\
		&\stackrel{(\eqref{eq:prob at t})}{=} \E_\bx \left(\sigma(\bw^\top \bx) - \sigma(\bv^\top \bx)\right) \onefunc(\bw^\top \bx \geq 0, \bv^\top \bx <0) \tilde{\bx}
		\\
		&= \E_\bx \left(\sigma(\bw^\top \bx)\right) \tilde{\bx}
		\\
		&= \E_{\tilde{\bx}} \onefunc(\tilde{\bw}^\top \tilde{\bx} > -b_\bw) (\tilde{\bw}^\top \tilde{\bx} + b_\bw )\tilde{\bx}~.
	\end{align*}
	If 
	$-\frac{b_\bw}{\norm{\tilde{\bw}}} \geq r$
	then for every $\tilde{\bx} \in \Bcal$ we have 
	$\tilde{\bw}^\top \tilde{\bx} \leq \norm{\tilde{\bw}} r \leq -b_\bw$
	and hence $\frac{d \tilde{\bw}}{dt}  = \zero$. 
	Note that if 
	$-\frac{b_\bw}{\norm{\tilde{\bw}}} < r$,
	i.e., 
	$\norm{\tilde{\bw}} r > -b_\bw$,
	then $\Pr_{\tilde{\bx}} \left[ \tilde{\bw}^\top \tilde{\bx} > -b_\bw \right] > 0$. Since $\tilde{\cd}$ is spherically symmetric, then we obtain $\frac{d \tilde{\bw}}{dt} = -s \tilde{\bw}$ for some $s > 0$.
	
	Next, we have
	\begin{align*}
		- \frac{d b_\bw}{dt}
		&= \nabla_{b_\bw}F(\bw)
		= \E_\bx  \left(\sigma(\bw^\top \bx) - \sigma(\bv^\top \bx)\right) \sigma'(\bw^\top \bx) \cdot 1
		\\
		&= \E_\bx \left(\sigma(\bw^\top \bx) - \sigma(\bv^\top \bx)\right) \onefunc(\bw^\top \bx \geq 0) 
		\\
		&\stackrel{(\eqref{eq:prob at t})}{=}  \E_\bx \left(\sigma(\bw^\top \bx) - \sigma(\bv^\top \bx)\right) \onefunc(\bw^\top \bx \geq 0, \bv^\top \bx <0) 
		\\
		&= \E_\bx \left(\sigma(\bw^\top \bx)\right) 
		\\
		&= \E_{\tilde{\bx}} \onefunc(\tilde{\bw}^\top \tilde{\bx} > -b_\bw) (\tilde{\bw}^\top \tilde{\bx} + b_\bw )~.
	\end{align*}
	If 
	$-\frac{b_\bw}{\norm{\tilde{\bw}}} \geq r$
	then for every $\tilde{\bx} \in \Bcal$ we have 
	$\tilde{\bw}^\top \tilde{\bx} \leq \norm{\tilde{\bw}} r \leq -b_\bw$
	and hence  $\frac{d b_\bw}{dt} = 0$. Otherwise, we have  $\frac{d b_\bw}{dt} < 0$.
\end{proof}

\begin{proof}[Proof of \thmref{thm:negative2}]
	By Lemma~\ref{lem:negative2 random init} $\bw_0$ satisfies $\Pr_{\bx \sim \cd} \left[\bw_0^\top \bx \geq 0, \bv^\top \bx \geq 0 \right] = 0$ w.p. at least $\frac{1}{2} - o_d(1)$.
	Then,
	by Lemma~\ref{lem:negative2 dynamics} we have for every $t>0$ that $\tilde{\bw}_t = \gamma_t \tilde{\bw}_0$ for some $\gamma_t > 0$, $b_{\bw_t}<0$, and 
	$-\frac{b_{\bw_t}}{\norm{\tilde{\bw}_t}} \leq r$. 
	Moreover, we have  $\Pr_{\bx \sim \cd} \left[\bw_t^\top \bx \geq 0, \bv^\top \bx \geq 0 \right] = 0$.
	Hence, for every $t$ we have 
	\begin{align*}
    		F(\bw_t)
 	   	&= \frac{1}{2} \cdot \E_\bx \left(\sigma(\bw_t^\top \bx) - \sigma(\bv^\top \bx)\right)^2
    		\\
  	  	&=  \frac{1}{2} \cdot \E_\bx \left(\sigma(\bw_t^\top \bx)\right)^2 +  \frac{1}{2} \cdot \E_\bx \left(\sigma(\bv^\top \bx)\right)^2 - \E_\bx  \left(\sigma(\bw_t^\top \bx ) \sigma(\bv^\top \bx)\right)
    		\\
   	 	&= \frac{1}{2} \cdot \E_\bx \left(\sigma(\bw_t^\top \bx)\right)^2 +  \frac{1}{2} \cdot \E_\bx \left(\sigma(\bv^\top \bx)\right)^2
    		\\
  	  	&\geq \frac{1}{2} \cdot \E_\bx \left(\sigma(\bv^\top \bx)\right)^2
    		= F(\zero)~.
	\end{align*}
  	Thus, gradient flow does not converge to the global minimum $F(\bv)=0<F(\zero)$.
\end{proof}

\section{Proofs from \secref{sec:characterization}}\label{appen:proofs from charac}
\begin{proof}[Proof of \thmref{thm:characterization}]
The gradient of the objective is:
\begin{equation*}
\nabla F(\bw) = 
\E_{\bx\sim\Dcal}\left[\left(\sigma(\bw^\top\bx)-\sigma(\bv^\top\bx)\right)
\cdot\sigma'(\bw^\top\bx)\bx\right].
\end{equation*}
We can rewrite it using that $\sigma$ is the ReLU activation, and separating the bias terms:
\begin{equation*}
    \nabla F(\bw) = 
    \E_{\tilde{\bx}\sim\tilde{\Dcal}}\left[\left(\sigma(\tilde{\bw}^\top\tilde{\bx} + b_\bw)-\sigma(\tilde{\bv}^\top\tilde{\bx} + b_\bv)\right)
    \cdot\mathbbm{1}(\tilde{\bw}^\top\tilde{\bx} + b_\bw > 0)\bx\right].
\end{equation*}
First, notice that if $\tilde{\bw}= 0 $ and $b_\bw <0$ then $\mathbbm{1}(\tilde{\bw}^\top\tilde{\bx} + b_\bw > 0) = 0$ for all $\tilde{\bx}$, hence $\nabla F(\bw) =0$. Second, using Cauchy-Schwartz we have that $|\inner{\tilde{\bw},\tilde{\bx}}| \leq c\cdot \norm{\tilde{\bw}}$. Hence, for $\bw$ with $\tilde{\bw} \neq 0$ and  $-\frac{b_\bw}{\norm{\tilde{\bw}}} \geq c$ we have that $\mathbbm{1}(\tilde{\bw}^\top\tilde{\bx} + b_\bw > 0) = 0$ for all $\tilde{\bx}$ in the support of the distribution, hence $\nabla F(\bw) =0 $. Lastly, it is clear that for $\bw = \bv$ we have that $\nabla F (\bw) = 0$. This shows that the points described in the statement of the proposition are indeed critical points. Next we will show that these are the only critical points.


Let $\bw\in\reals^{d+1}$ which is not a critical point defined above - i.e. either $\tilde{\bw} = \bm{0}$ and $b_\bw >0$, or $\tilde{\bw}\neq \bm{0}$ and $-\frac{b_\bw}{\norm{\tilde{\bw}}} < c$. Then we have:
\begin{align}\label{eq:only local minima}
    \inner{\nabla F(\bw),\bw-\bv} &= \mathbb{E}_{\bx\sim \cd}\left[\sigma'(\bw^\top \bx)(\sigma(\bw^\top \bx) - \sigma(\bv^\top \bx))(\bw^\top \bx - \bv^\top \bx) \right] \nonumber\\
    & = \mathbb{E}_{\bx\sim \cd} \left[ \mathbbm{1}(\bw^\top \bx >0, \bv^\top \bx >0) (\sigma(\bw^\top \bx) - \sigma(\bv^\top \bx))(\bw^\top \bx - \bv^\top \bx) \right] + \nonumber \\
    & + \mathbb{E}_{\bx\sim \cd} \left[ \mathbbm{1}(\bw^\top \bx >0, \bv^\top \bx \leq 0) \sigma(\bw^\top \bx)(\bw^\top \bx - \bv^\top \bx) \right] \nonumber \\
    & \geq \mathbb{E}_{\bx\sim \cd} \left[ \mathbbm{1}(\bw^\top \bx >0, \bv^\top \bx >0) (\bw^\top \bx - \bv^\top \bx)^2 \right] +  \nonumber\\
    & +  \mathbb{E}_{\bx\sim \cd} \left[ \mathbbm{1}(\bw^\top \bx >0, \bv^\top \bx \leq 0) (\bw^\top \bx)^2  \right]~.\nonumber \\
    & = \mathbb{E}_{\bx\sim \cd} \left[ \mathbbm{1}(\tilde{\bw}^\top \tilde{\bx} >-b_\bw, \tilde{\bv}^\top \tilde{\bx} >-b_\bv) (\bw^\top \bx - \bv^\top \bx)^2 \right] +  \nonumber\\
    & +  \mathbb{E}_{\bx\sim \cd} \left[ \mathbbm{1}(\tilde{\bw}^\top \tilde{\bx} >-b_\bw, \tilde{\bv}^\top \tilde{\bx} \leq -b_\bv) (\bw^\top \bx)^2  \right]~.
\end{align}
Denote:
\begin{align*}
    A_1 &:= \{\tilde{\bx}\in\reals^{d}:\tilde{\bw}^\top\tilde{\bx} >-b_\bw, \tilde{\bv}^\top \tilde{\bx} >-b_\bv, \norm{\tilde{\bx}} < c  \} \\
    A_2 &:= \{\tilde{\bx}\in\reals^{d}:\tilde{\bw}^\top\tilde{\bx} >-b_\bw, \tilde{\bv}^\top \tilde{\bx} \leq -b_\bv, \norm{\tilde{\bx}} < c \}
\end{align*}
Since $\bw$ is not a critical point as defined above, we know that the set $\{\tilde{\bx}\in\reals^{d}:\tilde{\bw}^\top\tilde{\bx} >-b_\bw, \norm{\tilde{\bx}} < c\}$ has a positive measure, hence either $A_1$ or $A_2$ have a positive measure. Assume w.l.o.g that $A_1$ have a positive measure, the other case is similar. Since both terms inside the expectations of \eqref{eq:only local minima} are positive, we can lower bound it with:

\begin{align}\label{eq:only local minima before int}
    &\mathbb{E}_{\bx\sim \cd} \left[ \mathbbm{1}(\tilde{\bw}^\top \tilde{\bx} >-b_\bw, \tilde{\bv}^\top \tilde{\bx} >-b_\bv) (\bw^\top \bx - \bv^\top \bx)^2 \right] \nonumber \\
    &=  \norm{\bw-\bv}^2  \mathbb{E}_{\bx\sim \cd} \left[ \mathbbm{1}(\tilde{\bx}\in A_1) ((\overline{\bw-\bv})^\top\bx)^2 \right]
\end{align}
Denote $\bu:= \overline{\bw-\bv}$, and note that  $\bw\neq \bv$, hence $\norm{\bu} =  1$. Denote by $p(\tilde{\bx})$ the pdf of $\tilde{\Dcal}$, then we can rewrite \eqref{eq:only local minima before int} as:
\begin{align}\label{eq:only local minima int}
    &\norm{\bw-\bv}^2 \cdot\int_{\tilde{\bx}\in\reals^d}\mathbbm{1}(\tilde{\bx}\in {A}_1)\cdot(\tilde{\bu}^\top\tilde{\bx} + b_\bu)^2 p(\tilde{\bx})d\tilde{\bx} \nonumber\\
    &=  \norm{\bw-\bv}^2 \cdot\int_{\tilde{\bx}\in A_1}(\tilde{\bu}^\top\tilde{\bx} + b_\bu)^2 p(\tilde{\bx})d\tilde{\bx}
\end{align}
Since the set $A_1$ has a positive measure, and the set $\{\tilde{\bx}:\tilde{\bu}^\top\tilde{\bx} + b_\bu =0\}$ is of zero measure, there is a point $\tilde{\bx}_0$ such that $\tilde{\bu}^\top\tilde{\bx} + b_\bu \neq 0$. By continuity, there is a small enough neighborhood $A$ of $\tilde{\bx}_0$, such that $\tilde{\bu}^\top\tilde{\bx} + b_\bu \neq 0$ for every $\tilde{\bx}\in A$. Using Assumption \ref{assum:characterization} we can lower bound \eqref{eq:only local minima int} by:
\[
\norm{\bw-\bv}^2 \cdot \beta \int_{\tilde{\bx}\in A}(\tilde{\bu}^\top\tilde{\bx} + b_\bu)^2 d\tilde{\bx}
\]
where this integral is positive. This shows that $\inner{\nabla F(\bw),\bw-\bv} >0$, which shows that $\nabla F(\bw)\neq \bm{0}$, hence $\bw$ is not a critical point.

\end{proof}

\section{Proofs from \secref{sec:init better than trivial}}
\label{appen:proofs from init better than trivial}

The following lemmas are required in order to prove Theorem~\ref{thm:smooth linear rate}. First, we show that if $F(\bw) \leq F(\zero) - \delta$ then we can lower bound $\norm{\bw}$ and $\Pr_\bx \left[ \bw^\top \bx \geq 0, \bv^\top \bx \geq 0 \right]$.

\begin{lemma}
\label{lem:prob and norm lower bound}
	Let $\delta>0$ and let $\bw \in \reals^{d+1}$ such that $F(\bw) \leq F(\zero) - \delta$. Then  
	\[
		\norm{\bw} \geq  \frac{\delta}{c^2}~,
	\]
	and 
	\[
		\Pr_\bx \left[ \bw^\top \bx \geq 0, \bv^\top \bx \geq 0 \right] \geq \frac{\delta}{c^2 \norm{\bw}}~.
	\] 
\end{lemma}
\begin{proof}
	We have
	\begin{align*}
		F(\zero) - \delta 
		&\geq F(\bw)
		= \frac{1}{2} \E_\bx (\sigma(\bw^\top \bx) - \sigma(\bv^\top \bx))^2
		\\
		&= \frac{1}{2} \E_\bx (\sigma(\bw^\top \bx))^2 +  \frac{1}{2} \E_\bx (\sigma(\bv^\top \bx))^2 - \E_\bx (\sigma(\bw^\top \bx) \sigma(\bv^\top \bx))
		\\
		&\geq  F(\zero) - \E_\bx (\sigma(\bw^\top \bx) \sigma(\bv^\top \bx))~.
	\end{align*}

	Hence
	\begin{align*}
		\delta 
		&\leq \E_\bx \sigma(\bw^\top \bx) \sigma(\bv^\top \bx) 
		=  \E_\bx \mathbbm{1}(\bw^\top \bx \geq 0,  \bv^\top \bx \geq 0) \cdot \bw^\top \bx \cdot \bv^\top \bx
		\\
		&\leq \norm{\bw} c^2 \cdot \Pr_\bx \left[\bw^\top \bx \geq 0,  \bv^\top \bx \geq 0\right]~.
	\end{align*}

	Thus, 
	\[
		\norm{\bw}
		\geq \frac{\delta}{c^2 \cdot \Pr_\bx \left[\bw^\top \bx \geq 0,  \bv^\top \bx \geq 0\right]}
		\geq  \frac{\delta}{c^2}~,
	\]
	and 
	\[
		\Pr_\bx \left[\bw^\top \bx \geq 0,  \bv^\top \bx \geq 0\right] 
		\geq \frac{\delta}{c^2 \norm{\bw}}~.
	\]
\end{proof}

Using the above lemma, we now show that if $F(\bw) \leq F(\zero) - \delta$ then $\norm{\bw-\bv}$ decreases.

\begin{lemma}
\label{lem:norm upper bound}
	Let $\delta > 0$ and let 
	$B > 1$.
	Let $\bw \in \reals^{d+1}$ such that $F(\bw) \leq F(\zero) - \delta$ and $\norm{\bw-\bv} \leq B-1$. Let $\gamma =  \frac{\delta^3}{3 \cdot 12^2B^3c^8c'^2}$ and let $0 < \eta \leq \frac{\gamma}{c^4}$. Let $\bw' = \bw - \eta \nabla F(\bw)$. Then, 
	\[
		\norm{\bw'-\bv}^2 \leq \norm{\bw - \bv}^2 \cdot \left(1 - \gamma \eta \right) \leq (B-1)^2~.
	\]  
\end{lemma}
\begin{proof}
	We have
	\begin{align}
	\label{eq:bound dist norm}
		\norm{\bw'-\bv}^2 
		&= \norm{\bw - \eta \nabla F(\bw) -\bv}^2 \nonumber
		\\
		&=\norm{\bw - \bv}^2 - 2 \eta \inner{\nabla F(\bw), \bw - \bv} + \eta^2 \norm{\nabla F(\bw)}^2~.
	\end{align}

	We first bound $ \norm{\nabla F(\bw)}^2$. By Jensen's inequality and since $\sigma$ is $1$-Lipschitz, we have:
	\begin{align}
	\label{eq:bound grad norm}
		 \norm{\nabla F(\bw)}^2
		 &\leq \E_\bx \left[ \left( \sigma(\bw^\top \bx) - \sigma(\bv^\top \bx) \right)^2 \sigma'(\bw^\top \bx) \norm{\bx}^2 \right] \nonumber
		 \\
		 &\leq c^2 \E_\bx  \left[ \left(\sigma(\bw^\top \bx) - \sigma(\bv^\top \bx) \right)^2  \right] \nonumber
		 \\
		 &\leq c^2 \E_\bx  \left[ \left(\bw^\top \bx - \bv^\top \bx \right)^2 \right] \nonumber
		 \\
		 &= c^2 \E_\bx  \left[ \left((\bw - \bv)^\top \bx \right)^2 \right] \nonumber
		 \\
		 &\leq c^4 \norm{\bw - \bv}^2~.
	\end{align} 

	Next, we bound $ \inner{\nabla F(\bw), \bw - \bv}$. Let $\bu = \overline{\bw-\bv}$. We have
	\begin{align*}
		 \inner{\nabla F(\bw), \bw - \bv} 
		 &= \E_\bx \left(\sigma(\bw^\top \bx) - \sigma(\bv^\top \bx)\right) \sigma'(\bw^\top \bx) (\bw^\top \bx - \bv^\top \bx)
		 \\
		 &= \E_\bx \left(\bw^\top \bx -\bv^\top \bx\right)^2  \mathbbm{1}(\bw^\top \bx \geq 0, \bv^\top \bx \geq 0) + 
		 \\
		 &\;\;\;\; \E_\bx \bw^\top \bx \cdot  (\bw^\top \bx - \bv^\top \bx) \mathbbm{1}(\bw^\top \bx \geq 0, \bv^\top \bx < 0) 
		 \\
		 &\geq \norm{\bw-\bv}^2 \cdot \E_\bx \mathbbm{1}(\bw^\top \bx \geq 0, \bv^\top \bx \geq 0) (\bu^\top \bx)^2~.
	\end{align*}
	
	Let $\xi =\frac{\delta}{12Bc^3c'}$. The above is at least
	\begin{align*}
		&\norm{\bw-\bv}^2 \cdot \xi^2 \cdot \Pr_\bx\left[ \bw^\top \bx \geq 0, \bv^\top \bx \geq 0, (\bu^\top \bx)^2 \geq \xi^2 \right]
		 \\
		 &= \norm{\bw-\bv}^2 \cdot \xi^2 \cdot \left(  \Pr_\bx\left[ \bw^\top \bx \geq 0, \bv^\top \bx \geq 0 \right] -  \Pr_\bx\left[ \bw^\top \bx \geq 0, \bv^\top \bx \geq 0, (\bu^\top \bx)^2 < \xi^2 \right]\right)~.
	\end{align*}

	By Lemma~\ref{lem:prob and norm lower bound}, and since $\norm{\bw} \leq \norm{\bw-\bv} + \norm{\bv} \leq B-1 + 1 = B$, the above is at least
	\begin{align}
	\label{eq:bound inner prod1}
		 \norm{\bw-\bv}^2 &\cdot \xi^2 \cdot \left( \frac{\delta}{c^2\norm{\bw}} -  \Pr_\bx\left[ \bw^\top \bx \geq 0, \bv^\top \bx \geq 0, |\bu^\top \bx| < \xi \right]\right) \nonumber
		 \\
		 &\geq  \norm{\bw-\bv}^2 \cdot \xi^2 \cdot \left( \frac{\delta}{c^2B} - \Pr_\bx\left[ |\bu^\top \bx| \leq \xi \right]\right) \nonumber
		 \\
		 &= \norm{\bw-\bv}^2 \cdot \xi^2 \cdot \left( \frac{\delta}{c^2B} - \Pr_\bx\left[ |\tilde{\bu}^\top \tilde{\bx} + b_\bu| \leq \xi \right]\right)~.
	\end{align}

	We now bound $\Pr_\bx\left[ |\tilde{\bu}^\top \tilde{\bx} + b_\bu| \leq \xi \right]$.
	We denote $a=\norm{\tilde{\bu}}$. If $a \leq \frac{1}{4c}$, then since $\norm{\bu}=1$ we have $b_\bu \geq \sqrt{1 - \frac{1}{16c^2}} \geq  \sqrt{1 - \frac{1}{16}} = \frac{\sqrt{15}}{4}$.
	Hence, for every $\bx$ with $\norm{\bx} \leq c$ we have 
	\[
		 |\tilde{\bu}^\top \tilde{\bx} + b_\bu| 
		 \geq |b_\bu| - | \tilde{\bu}^\top \tilde{\bx} |
		 \geq \frac{\sqrt{15}}{4} - ac
		 \geq \frac{\sqrt{15}}{4} - \frac{1}{4}
		 > \frac{1}{2}~.
	\]
	Note that 
	 \[
	 	\xi = \frac{\delta}{12Bc^3c'} 
		\leq \frac{F(\zero)}{12Bc^3c'}
		=  \frac{1}{12Bc^3c'} \cdot \frac{1}{2} \E_\bx(\sigma(\bv^\top \bx))^2
		\leq  \frac{1}{12Bc^3c'} \cdot \frac{1}{2} c^2
		=  \frac{1}{24Bcc'} 
		\leq \frac{1}{24}~,
	 \]
	where the last inequality is since $B,c,c' \geq 1$.
	Therefore, $ |\tilde{\bu}^\top \tilde{\bx} + b_\bu| > \xi$.
	Thus,
	\[
		 \Pr_\bx\left[ |\tilde{\bu}^\top \tilde{\bx} + b_\bu| \leq \xi \right] = 0~.
	\]

	Assume now that $a \geq \frac{1}{4c}$. We have 
	\begin{align*}
		 \Pr_\bx\left[ |\tilde{\bu}^\top \tilde{\bx} + b_\bu| \leq \xi \right] 
		 &=  \Pr_\bx\left[ \tilde{\bu}^\top \tilde{\bx} \in [-\xi - b_\bu, \xi - b_\bu] \right] 
		 \\
		 &= \Pr_\bx\left[ \bar{\tilde{\bu}}^\top \tilde{\bx} \in [-\frac{\xi}{a} - \frac{b_\bu}{a}, \frac{\xi}{a} - \frac{b_\bu}{a}] \right]
		 \\
		 &\leq c' \cdot 2 \cdot \frac{\xi}{a} 
		 \\
		 &\leq 8cc'\xi~.
	\end{align*}

	Combining the above with \eqref{eq:bound inner prod1}, we obtain 
	\begin{align}
	\label{eq:bound inner prod2}
		\inner{\nabla F(\bw), \bw - \bv}
		&\geq  \norm{\bw-\bv}^2 \cdot \xi^2 \cdot \left( \frac{\delta}{c^2B} - 8cc'\xi \right) \nonumber
		\\
		&=  \norm{\bw-\bv}^2 \frac{\delta^2}{12^2B^2c^6c'^2}  \cdot \left( \frac{\delta}{c^2B} - 8cc' \cdot \frac{\delta}{12Bc^3c'} \right) \nonumber
		\\
		&=  \norm{\bw-\bv}^2 \frac{\delta^2}{12^2B^2c^6c'^2}  \cdot \left( \frac{\delta}{3c^2B}  \right) \nonumber
		\\
		&=  \norm{\bw-\bv}^2 \frac{\delta^3}{3 \cdot 12^2B^3c^8c'^2}~.
	\end{align}
	
	Combining \eqref{eq:bound dist norm},~(\ref{eq:bound grad norm}) and~(\ref{eq:bound inner prod2}), and using $\gamma =  \frac{\delta^3}{3 \cdot 12^2B^3c^8c'^2}$, we have
	\begin{align*}
		\norm{\bw'-\bv}^2 
		&\leq \norm{\bw - \bv}^2 - 2 \eta \norm{\bw-\bv}^2 \cdot \gamma + \eta^2 c^4 \norm{\bw - \bv}^2
		\\
		&= \norm{\bw - \bv}^2 \cdot \left(1 - 2 \eta \gamma + \eta^2 c^4  \right)~.
	\end{align*}
	Since $\eta \leq \frac{\gamma}{c^4}$, we obtain
	\begin{align*}
		\norm{\bw'-\bv}^2 
		&\leq  \norm{\bw - \bv}^2 \cdot \left(1 - 2 \eta \gamma + \eta c^4 \cdot  \frac{\gamma}{c^4} \right)
		\\
		&=  \norm{\bw - \bv}^2 \cdot \left(1 - \gamma \eta \right) \leq  \norm{\bw - \bv}^2 \leq (B-1)^2~.
	\end{align*}
\end{proof}

Next, we show that $F(\bw)$ remains smaller than $F(\zero)-\delta$ during the training. In the following two lemmas we obtain a bound for the smoothness of $F$ in the relevant region, and in the two lemmas that follow we use this bound to show that $F(\bw)$ indeed remains small.

\begin{lemma}
\label{lem:grad upper bound}
	Let $\bw \in \reals^{d+1}$ such that $F(\bw) \leq F(\zero)$. Then, $\norm{\nabla F(\bw)} \leq c \sqrt{2F(\zero)}$.
\end{lemma}
\begin{proof}
	By Jensen's inequality, we have
	\begin{align*}
		\norm{\nabla F(\bw)}^2
		&\leq \E_\bx \left(\sigma(\bw^\top \bx) - \sigma(\bv^\top \bx) \right)^2 \sigma'(\bw^\top \bx) \norm{\bx}^2
		\\
		&\leq c^2 \E_\bx \left(\sigma(\bw^\top \bx) - \sigma(\bv^\top \bx) \right)^2 
		\\
		&\leq c^2 2 F(\bw)
		\leq 2c^2 F(\zero)~.
	\end{align*}
\end{proof}

\begin{lemma}
\label{lem:smooth}
	Let $M,B>0$ and let $\bw,\bw' \in \reals^{d+1}$ be such that for every $s \in [0,1]$ we have $M \leq \norm{\bw + s (\bw' - \bw)} \leq B$. Then, 
	\[
		 \norm{\nabla F(\bw) - \nabla F(\bw')} \leq  \norm{\bw-\bw'} \cdot c^2 \left( 1 +  \frac{8(B+1)c'c^2}{M} \right)~.
	\]
\end{lemma}
\begin{proof}
	We assume w.l.o.g. that $\norm{\bw - \bw'} \leq \frac{M}{2c}$. Indeed, let $0 = s_0 < \ldots < s_k = 1$ for some integer $k$, let $\bw_i =  \bw+s_i(\bw' - \bw)$, and assume that $\norm{\bw_i - \bw_{i+1}} \leq  \frac{M}{2c}$ for every $i$. If the claim holds for every pair $\bw_i,\bw_{i+1}$, then we have 
	\begin{align*}
		 \norm{\nabla F(\bw) - \nabla F(\bw')} 
		 &= \norm{ \sum_{i=0}^{k-1}  \nabla F(\bw_i) - \nabla F(\bw_{i+1})  }
		 \\
		 &\leq  \sum_{i=0}^{k-1} \norm{\nabla F(\bw_i) - \nabla F(\bw_{i+1}) }
		 \\
		 &\leq \sum_{i=0}^{k-1}  \norm{\bw_i - \bw_{i+1}} \cdot c^2 \left( 1 +  \frac{8(B+1)c'c^2}{M} \right)
		 \\
		 &= c^2 \left( 1 +  \frac{8(B+1)c'c^2}{M} \right) \norm{\bw-\bw'}~.
	\end{align*}

	We have
	\begin{align*}
		\norm{\nabla F(\bw) &- \nabla F(\bw')}
		\\
		&= \norm{\E_\bx (\sigma(\bw^\top \bx) - \sigma(\bv^\top \bx))\sigma'(\bw^\top \bx) \bx - (\sigma(\bw'^\top \bx) - \sigma(\bv^\top \bx))\sigma'(\bw'^\top \bx) \bx}
		\\
		&\leq \norm{\E_\bx \mathbbm{1}(\bw^\top \bx \geq 0, \bw'^\top \bx \geq 0) \left(\bw^\top \bx - \sigma(\bv^\top \bx) - \bw'^\top \bx + \sigma(\bv^\top \bx)\right)\bx} +
		\\
		&\;\;\;\;\, \norm{\E_\bx \mathbbm{1}(\bw^\top \bx \geq 0, \bw'^\top \bx < 0) \left(\bw^\top \bx - \sigma(\bv^\top \bx) \right)\bx} +
		\\
		&\;\;\;\;\, \norm{\E_\bx\mathbbm{1}(\bw^\top \bx < 0, \bw'^\top \bx \geq 0) \left(\bw'^\top \bx - \sigma(\bv^\top \bx) \right)\bx}~.
	\end{align*}
	By Jensen's inequality and Cauchy-Shwartz, the above is at most
	\begin{align*}	
		&\E_\bx \mathbbm{1}(\bw^\top \bx \geq 0, \bw'^\top \bx \geq 0)  \norm{\bw-\bw'}  \cdot \norm{\bx} \cdot \norm{\bx} +
		\\
		&\E_\bx \mathbbm{1}(\bw^\top \bx \geq 0, \bw'^\top \bx < 0) \left(\norm{\bw} \cdot \norm{\bx} + \norm{\bv} \cdot \norm{\bx} \right) \cdot \norm{\bx} + 
		\\
		&\E_\bx \mathbbm{1}(\bw^\top \bx < 0, \bw'^\top \bx \geq 0) \left(\norm{\bw'} \cdot  \norm{\bx} + \norm{\bv} \cdot \norm{\bx} \right) \cdot \norm{\bx}~.
	\end{align*}
	By our assumption we have $\norm{\bx} \leq c$ and $ \norm{\bw},\norm{\bw'} \leq B$. Hence, the above is at most
	\begin{align}
	\label{eq:smooth}
		\norm{\bw&-\bw'} c^2 + 
		\Pr_\bx\left[\bw^\top \bx \geq 0, \bw'^\top \bx < 0\right] \cdot c^2 \cdot (B+1)  \nonumber
		\\
		& + \Pr_\bx\left[\bw^\top \bx < 0, \bw'^\top \bx \geq 0\right]  \cdot c^2 \cdot (B+1)~.
	\end{align}
	
	Now, we bound $\Pr_\bx\left[\bw^\top \bx \geq 0, \bw'^\top \bx < 0\right]$. 
	 If $\bw^\top \bx \geq 0$ and $\bw'^\top \bx < 0$ then 
	 \[
	 	\bw^\top \bx = \bw'^\top \bx + (\bw-\bw')^\top \bx < 0 + \norm{\bw-\bw'} \cdot \norm{\bx} \leq c \cdot \norm{\bw-\bw'}~.
	\]
	Hence, we only need to bound 
	\begin{align*}
		\Pr_\bx\left[\bw^\top \bx \in [0, c \cdot \norm{\bw-\bw'}]\right] 
		= \Pr_\bx\left[\tilde{\bw}^\top \tilde{\bx} + b_\bw \in [0, c \cdot \norm{\bw-\bw'}]\right]~. 
	\end{align*}
	We denote $a=\norm{\tilde{\bw}}$.
	If $a \leq \frac{M}{4c}$, then since $\norm{\bw} \geq M$ we have $|b_\bw| \geq \sqrt{M^2 - \left(\frac{M}{4c}\right)^2} = M \sqrt{1-1/(16c^2)}$. Hence for every $\bx$ we have 
	\begin{align*}
		|\tilde{\bw}^\top \tilde{\bx} + b_\bw| 
		&\geq | b_\bw| - |\tilde{\bw}^\top \tilde{\bx}|  
		\geq | b_\bw| - ac 
		\geq M \sqrt{1-1/(16c^2)} -  \frac{M}{4} 
		\geq M \sqrt{1-1/16} -  \frac{M}{4} 
		\\
		&= M \cdot \frac{\sqrt{15}-1}{4} 
		> M/2 
		\geq c \norm{\bw-\bw'}~. 
	\end{align*}
	Thus, $\Pr_\bx \left[\tilde{\bw}^\top \tilde{\bx} + b_\bw \in [0,c \cdot \norm{\bw-\bw'}]\right] = 0$.
	
	Assume now that $a \geq \frac{M}{4c}$. Hence, $\frac{c}{a} \leq \frac{4c^2}{M}$. 
	Therefore, we have 
	\begin{align*}
		\Pr_\bx\left[\tilde{\bw}^\top \tilde{\bx} + b_\bw \in [0, c \cdot \norm{\bw-\bw'}]\right] 
		&= \Pr_\bx\left[\bar{\tilde{\bw}}^\top \tilde{\bx} \in [-\frac{b_\bw}{a}, -\frac{b_\bw}{a} + \frac{c}{a} \cdot \norm{\bw-\bw'}]\right] 
		\\
		&\leq \Pr_\bx\left[\bar{\tilde{\bw}}^\top \tilde{\bx}  \in [-\frac{b_\bw}{a}, -\frac{b_\bw}{a} +  \frac{4c^2}{M} \cdot  \norm{\bw-\bw'}]\right] 
		\\
		&\leq c' \cdot \frac{4c^2}{M} \cdot  \norm{\bw-\bw'}~.
	\end{align*}
	
	Hence, $\Pr_\bx\left[\bw^\top \bx \geq 0, \bw'^\top \bx < 0\right] \leq c' \cdot \frac{4c^2}{M} \cdot  \norm{\bw-\bw'}$.
	By similar arguments, this inequality holds also for $\Pr_\bx\left[\bw^\top \bx < 0, \bw'^\top \bx \geq 0\right]$. Plugging it into \eqref{eq:smooth}, we have
	\begin{align*}
		\norm{\nabla F(\bw) - \nabla F(\bw')} 
		&\leq \norm{\bw-\bw'} \left( c^2 + 2 \cdot c^2 \cdot (B+1) \cdot c' \cdot \frac{4c^2}{M} \right)
		\\
		&= \norm{\bw-\bw'} \cdot c^2 \left( 1 +  \frac{8(B+1)c'c^2}{M} \right)~.
	\end{align*}
\end{proof}

\begin{lemma}
\label{lem:beta smooth prop}
	Let $f: \reals^d \rightarrow \reals$ and let $L>0$. Let $\bx,\by \in \reals^d$ be such that for every $s \in [0,1]$ we have $\norm {\nabla f(\bx + s(\by-\bx)) - \nabla f(\bx)} \leq L s \norm{\by - \bx}$. Then, 
	\[
		f(\by) - f(\bx) 
		\leq  \nabla f(\bx)^\top (\by - \bx) +  \frac{L}{2} \norm{\by - \bx}^2~.
	\]
\end{lemma}
\begin{proof}
	The proof follows a standard technique (cf. \cite{bubeck2014convex}).
	We represent $f(\by) - f(\bx)$ as an integral, apply Cauchy-Schwarz and then use the $L$-smoothness.
	\begin{align*}
		f(\by) - f(\bx) - \nabla f(\bx)^\top (\by - \bx)
		&= \int_0^1 \nabla f(\bx + s(\by-\bx))^\top (\by - \bx) ds - \nabla f(\bx)^\top (\by - \bx)
		\\
		&\leq  \int_0^1 \norm {\nabla f(\bx + s(\by-\bx)) - \nabla f(\bx)} \cdot \norm{\by - \bx} ds 
		\\
		&\leq  \int_0^1 L s \norm{\by - \bx}^2 ds
		\\
		&= \frac{L}{2} \norm{\by - \bx}^2~.
	\end{align*}

	Hence, we have
	\begin{align*}
		f(\by) - f(\bx) 
		\leq  \nabla f(\bx)^\top (\by - \bx) +  \frac{L}{2} \norm{\by - \bx}^2~.
	\end{align*}
\end{proof}

\begin{lemma}
\label{lem:loss decrease}
	Let $B,\delta>0$ and let $L = c^2 \left( 1 +  \frac{16(B+1)c'c^4}{\delta} \right)$.
	Let $\bw \in \reals^{d+1}$ such that $F(\bw) \leq F(\zero) - \delta$ and let $\bw' = \bw - \eta \cdot \nabla F(\bw)$, where $\eta \leq \min \left\{\frac{\delta}{2c^3  \sqrt{2F(\zero)}}, \frac{1}{L} \right\} = \min \left\{\frac{\delta}{2c^3  \sqrt{2F(\zero)}}, \frac{\delta}{\delta c^2 + 16(B+1)c'c^6} \right\}$. 
	Assume that $\norm{\bw},\norm{\bw'} \leq B$.
	Then, we have $F(\bw') - F(\bw)  \leq  -\eta \left(1 - \frac{L}{2} \eta  \right) \norm{ \nabla F(\bw)}^2$, and $F(\bw') \leq F(\bw) \leq F(\zero) - \delta$.
\end{lemma} 
\begin{proof}
	Let $M =  \frac{\delta}{2c^2}$.
	By Lemmas~\ref{lem:prob and norm lower bound} and~\ref{lem:grad upper bound}, we have $\norm{\bw}  \geq  \frac{\delta}{c^2}$ and $\norm{\nabla F(\bw)} \leq c \sqrt{2F(\zero)}$. Hence for every $\lambda \in [0,1]$ we have 
	\[
	\norm{ \bw - \lambda \eta \nabla F(\bw)} 
	\geq  \frac{\delta}{c^2} - \eta \cdot c \sqrt{2F(\zero)}
	\geq  \frac{\delta}{c^2} -\frac{\delta}{2c^3  \sqrt{2F(\zero)}} \cdot c \sqrt{2F(\zero)}
	= \frac{\delta}{2c^2}  
	=M~.
	\]
	Since $\norm{\bw},\norm{\bw'} \leq B$, we also have $\norm{ \bw - \lambda \eta \nabla F(\bw)} \leq B$. 	
	By Lemma~\ref{lem:smooth}, we have for every $\lambda \in [0,1]$ that 
	\[
		 \norm{\nabla F(\bw) - \nabla F( \bw - \lambda \eta \nabla F(\bw))} \leq  \lambda \eta \norm{\nabla F(\bw)} \cdot c^2 \left( 1 +  \frac{8(B+1)c'c^2}{M} \right)~.
	\]
	We have $L = c^2 \left( 1 +  \frac{16(B+1)c'c^4}{\delta} \right) = c^2 \left( 1 +  \frac{8(B+1)c'c^2}{M} \right)$.
	 By Lemma~\ref{lem:beta smooth prop} we have 
	 \[
	 	F(\bw - \eta \nabla F(\bw))- F(\bw)  
	 	\leq - \eta \norm{\nabla F(\bw)}^2 + \frac{L}{2} \eta^2  \norm{\nabla F(\bw)}^2~.
	 \]
	 Since $\eta \leq \frac{1}{L}$, we also have $	F(\bw - \eta \nabla F(\bw)) \leq F(\bw)  \leq F(\zero) - \delta$.
\end{proof}

We are now ready to prove the theorem:

\begin{proof}[Proof of Theorem~\ref{thm:smooth linear rate}]
	Let $B =\norm{\bw_0}+2$.
	Assume that $\eta \leq \min \left\{\frac{\delta}{2c^3  \sqrt{2F(\zero)}}, \frac{\delta}{\delta c^2 + 16(B+1)c'c^6}, \frac{\gamma}{c^4} \right\}$.
	We have $\norm{\bw_0 - \bv} \leq \norm{\bw_0} + \norm{\bv} = \norm{\bw_0} + 1 \leq B-1$. By Lemmas~\ref{lem:norm upper bound} and~\ref{lem:loss decrease}, for every $t$ we have $\norm{\bw_t - \bv} \leq B-1$ (thus, $\norm{\bw_t} \leq B$) and  $F(\bw_t) \leq F(\zero) - \delta$.
	Moreover, by Lemma~\ref{lem:norm upper bound}, we have for every $t$ that $\norm{\bw_{t+1}-\bv}^2 \leq \norm{\bw_t - \bv}^2 \cdot \left(1 - \gamma \eta \right)$.  Therefore, $\norm{\bw_t - \bv}^2 \leq \norm{\bw_0 - \bv}^2 \left(1-\gamma \eta \right)^t$.
	
	It remains to show that 
	\[
		 \min \left\{\frac{\delta}{2c^3  \sqrt{2F(\zero)}}, \frac{\delta}{\delta c^2 + 16(B+1)c'c^6}, \frac{\gamma}{c^4} \right\} =  \frac{\gamma}{c^4}~.
	\]
	Note that we have $\delta \leq F(\zero) = \frac{1}{2}\E_\bx(\sigma(\bv^\top \bx))^2 \leq \frac{1}{2} \cdot c^2$. Thus
	\[
		\frac{\gamma}{c^4} 
		= \frac{\delta^3}{3 \cdot 12^2B^3c^{12}c'^2}
		\leq \frac{\delta}{3 \cdot 12^2B^3c^{12}c'^2} \cdot \frac{c^4}{4}
		= \frac{\delta}{12^3B^3c^8c'^2}~.		
	\]
	We have
	\[
		\frac{\delta}{2c^3  \sqrt{2F(\zero)}}
		\geq \frac{\delta}{2c^4}
		\geq \frac{\gamma}{c^4}~,
	\]
	where the last inequality is since $B,c,c' \geq 1$.
	Finally,
	\[
		\frac{\delta}{\delta c^2 + 16(B+1)c'c^6}
		\geq \frac{\delta}{\frac{c^4}{2} + 16(B+1)c'c^6}
		\geq \frac{\delta}{17(B+1)c'c^6}
		\geq  \frac{\delta}{34Bc'c^6}
		\geq \frac{\gamma}{c^4}~.
	\]
\end{proof}

\subsection{Proofs from \subsecref{subsec:random init}}
\label{appen:proofs from random init}

\subsubsection*{Proof of \thmref{thm:random init}}

	We have
	\begin{align}
	\label{eq:random init main F}
	F(\bw)
	&= \frac{1}{2}\E_\bx \left(\sigma(\bw^\top \bx) - \sigma(\bv^\top \bx)\right)^2 \nonumber
	\\
	&= F(\zero) +  \frac{1}{2}\E_\bx \left(\sigma(\bw^\top \bx)\right)^2 - \E_\bx \sigma(\bw^\top \bx) \sigma(\bv^\top \bx) \nonumber
	\\
	&\leq  F(\zero) +  \frac{\norm{\bw}^2 c^2}{2} - \norm{\bw} \E_\bx \sigma(\bar{\bw}^\top \bx) \sigma(\bv^\top \bx)~.
	\end{align}
	
	Let $\xi = \frac{\alpha}{4\sqrt{c}} \sin \left( \frac{\pi}{8}\right)$. We have
	\begin{align}
	\label{eq:random init main xi2}
		\E_\bx \sigma(\bar{\bw}^\top \bx) \sigma(\bv^\top \bx)
		&\geq \xi^2 \cdot \Pr_\bx\left[  \sigma(\bar{\bw}^\top \bx) \sigma(\bv^\top \bx) \geq \xi^2  \right]  \nonumber
		\\
		&\geq \xi^2 \cdot \Pr_\bx\left[  \bar{\bw}^\top \bx \geq 2\sqrt{c} \xi, \bv^\top \bx \geq \frac{\xi}{2\sqrt{c}}  \right]~.
	\end{align}
	
	In the following two lemmas we 
	bound $\Pr_\bx\left[   \bar{\bw}^\top \bx \geq 2\sqrt{c} \xi, \bv^\top \bx \geq \frac{\xi}{2\sqrt{c}}  \right]$. 
	
	\begin{lemma}
	\label{lem:bv negative}
		If $b_\bv \geq 0$ then 
		\[
			\Pr_\bx\left[  \bar{\bw}^\top \bx \geq 2\sqrt{c} \xi, \bv^\top \bx \geq \frac{\xi}{2\sqrt{c}}  \right]
			\geq  \frac{\beta \left(\alpha\sin\left(\frac{\pi}{8}\right) -  2\sqrt{c} \xi \right)^2}{4\sin\left(\frac{\pi}{8}\right)}~.
		\]
	\end{lemma}
	\begin{proof}
	
	If $\norm{\tilde{\bv}} \geq \frac{1}{4c}$, then we have 
	\begin{align*}
	\label{eq:random init M2}
		\Pr_\bx\left[  \bar{\bw}^\top \bx \geq 2\sqrt{c} \xi, \bv^\top \bx \geq \frac{\xi}{2\sqrt{c}}  \right]
		&\geq \Pr_\bx \left[  \bar{\bw}^\top \bx \geq  2\sqrt{c} \xi, \tilde{\bv}^\top \tilde{\bx} \geq \frac{\xi}{2\sqrt{c}}  \right] \nonumber
		\\
		&= \Pr_\bx \left[  \tilde{\bar{\bw}}^\top \tilde{\bx} \geq  2\sqrt{c} \xi, \bar{\tilde{\bv}}^\top \tilde{\bx} \geq \frac{\xi}{2\sqrt{c}\norm{\tilde{\bv}}}  \right] \nonumber
		\\
		&\geq \Pr_\bx \left[  \tilde{\bar{\bw}}^\top \tilde{\bx} \geq 2\sqrt{c} \xi, \bar{\tilde{\bv}}^\top \tilde{\bx} \geq 2\sqrt{c} \xi  \right] \nonumber
		\\
		&\geq \frac{\beta \left(\alpha\sin\left(\frac{\pi}{8}\right) -  2\sqrt{c} \xi \right)^2}{4\sin\left(\frac{\pi}{8}\right)}~,
	\end{align*}
	where the last inequality is due to Lemma~\ref{lem:vol of P}, since  $\theta(\tilde{\bw},\tilde{\bv}) \leq \frac{3\pi}{4}$.
	
	If $\norm{\tilde{\bv}} \leq \frac{1}{4c}$, then 
	\[
		b_\bv \geq \sqrt{1-\frac{1}{16c^2}} \geq  \sqrt{1-\frac{1}{16}} = \frac{\sqrt{15}}{4} > \frac{3}{4}~,
	\] 
	and hence 
	\[
		\bv^\top \bx = \tilde{\bv}^\top \tilde{\bx} + b_\bv > -\frac{1}{4c} \cdot c +  \frac{3}{4} = \frac{1}{2} \geq \xi \geq \frac{\xi}{2\sqrt{c}}~.
	\] 
	Therefore,  
	\begin{equation*}
		\Pr_\bx\left[  \bar{\bw}^\top \bx \geq 2\sqrt{c} \xi, \bv^\top \bx \geq \frac{\xi}{2\sqrt{c}}  \right]
		= \Pr_\bx\left[  \bar{\bw}^\top \bx \geq 2\sqrt{c} \xi \right]
		= \Pr_\bx\left[  \tilde{\bar{\bw}}^\top \tilde{\bx} \geq 2\sqrt{c} \xi \right]~.
	\end{equation*}
	For $\tilde{\bu} \in \reals^d$ such that $\norm{\tilde{\bu}}=1$ and $\theta(\tilde{\bw},\tilde{\bu})= \frac{3\pi}{4}$,  Lemma~\ref{lem:vol of P} implies that the above is at least
	\begin{equation*}
	\label{eq:random init M1}
		\Pr_\bx\left[  \tilde{\bar{\bw}}^\top \tilde{\bx} \geq 2\sqrt{c} \xi, \tilde{\bu}^\top \tilde{\bx} \geq 2\sqrt{c} \xi\right]
		\geq  \frac{\beta \left(\alpha\sin\left(\frac{\pi}{8}\right) -  2\sqrt{c} \xi \right)^2}{4\sin\left(\frac{\pi}{8}\right)}~.		
	\end{equation*}
	\end{proof}

	\begin{lemma}
	\label{lem:bv positive}
		If $b_\bv < 0$ and $-\frac{b_\bv}{\norm{\tilde{\bv}}} \leq \alpha \cdot \frac{\sin\left(\frac{\pi}{8}\right)}{4}$, then 
		\[
			\Pr_\bx\left[  \bar{\bw}^\top \bx \geq 2\sqrt{c} \xi, \bv^\top \bx \geq \frac{\xi}{2\sqrt{c}}  \right]
			\geq  \frac{\beta \left(\alpha\sin\left(\frac{\pi}{8}\right) -  2\sqrt{c} \xi \right)^2}{4\sin\left(\frac{\pi}{8}\right)}~.	
		\]
	\end{lemma}
	\begin{proof}
	\begin{align}
	\label{eq:random init M2 b}
		\Pr_\bx\left[  \bar{\bw}^\top \bx \geq 2\sqrt{c} \xi, \bv^\top \bx \geq \frac{\xi}{2\sqrt{c}}  \right]
		&= \Pr_\bx \left[  \bar{\bw}^\top \bx \geq  2\sqrt{c} \xi, \tilde{\bv}^\top \tilde{\bx} \geq \frac{\xi}{2\sqrt{c}} - b_\bv \right] \nonumber
		\\
		&= \Pr_\bx \left[  \tilde{\bar{\bw}}^\top \tilde{\bx} \geq  2\sqrt{c} \xi, \bar{\tilde{\bv}}^\top \tilde{\bx} \geq \frac{\xi}{2\sqrt{c}\norm{\tilde{\bv}}} - \frac{b_\bv}{\norm{\tilde{\bv}}}  \right]~.
	\end{align}
	Moreover, we have
	\begin{align*}
		\left( \alpha \cdot \frac{\sin\left(\frac{\pi}{8}\right)}{4}\right)^2
		\geq \left(\frac{b_\bv}{\norm{\tilde{\bv}}} \right)^2
		= \frac{1-\norm{\tilde{\bv}}^2}{\norm{\tilde{\bv}}^2}
		= \frac{1}{\norm{\tilde{\bv}}^2} - 1~,
	\end{align*}
	and hence
	\begin{align*}
		\norm{\tilde{\bv}}^2
		\geq \frac{16}{\alpha^2  \sin^2 \left(\frac{\pi}{8}\right) + 16}
		\geq \frac{16}{\left(\alpha  \sin \left(\frac{\pi}{8}\right) + 4 \right)^2}
		\geq \frac{16}{\left(c \cdot 1 + 4c \right)^2}~,
	\end{align*}
	where in the last inequality we used $c \geq \alpha$ and $c \geq 1$.
	Thus,
	\[
		\norm{\tilde{\bv}} \geq \frac{4}{5c} \geq \frac{1}{2c}~. 
	\]
	Combining the above with \eqref{eq:random init M2 b}, and using $-\frac{b_\bv}{\norm{\tilde{\bv}}} \leq \alpha \cdot \frac{\sin\left(\frac{\pi}{8}\right)}{4}$, we have
	\begin{align*}
		\Pr_\bx\left[  \bar{\bw}^\top \bx \geq 2\sqrt{c} \xi, \bv^\top \bx \geq \frac{\xi}{2\sqrt{c}}  \right]
		&\geq \Pr_\bx \left[  \tilde{\bar{\bw}}^\top \tilde{\bx} \geq  2\sqrt{c} \xi, \bar{\tilde{\bv}}^\top \tilde{\bx} \geq \sqrt{c}\xi + \alpha \cdot \frac{\sin\left(\frac{\pi}{8}\right)}{4}  \right]
		\\
		&= \Pr_\bx \left[  \tilde{\bar{\bw}}^\top \tilde{\bx} \geq  2\sqrt{c} \xi, \bar{\tilde{\bv}}^\top \tilde{\bx} \geq 2\sqrt{c}\xi \right]
		\\
		&\geq \frac{\beta \left(\alpha\sin\left(\frac{\pi}{8}\right) -  2\sqrt{c} \xi \right)^2}{4\sin\left(\frac{\pi}{8}\right)}~,
	\end{align*}
	where the last inequality is due to Lemma~\ref{lem:vol of P}, since  $\theta(\tilde{\bw},\tilde{\bv}) \leq \frac{3\pi}{4}$.
	\end{proof}
	
	Combining \eqref{eq:random init main xi2}
	with Lemmas~\ref{lem:bv negative} and~\ref{lem:bv positive},
	we have
	\begin{align*}
		\E_\bx \sigma(\bar{\bw}^\top \bx) \sigma(\bv^\top \bx)
		&\geq \xi^2 \cdot \frac{\beta \left(\alpha\sin\left(\frac{\pi}{8}\right) -  2\sqrt{c} \xi \right)^2}{4\sin\left(\frac{\pi}{8}\right)}
		=  \frac{\alpha^2 \sin^2\left(\frac{\pi}{8}\right)}{16c} \cdot \frac{\beta \left(\frac{\alpha}{2}\sin\left(\frac{\pi}{8}\right) \right)^2}{4\sin\left(\frac{\pi}{8}\right)}
		\\
		&= \frac{\alpha^4 \beta \sin^3 \left(\frac{\pi}{8}\right)}{256c} 
		= M~.
	\end{align*}
	Plugging the above into \eqref{eq:random init main F} we have
	\begin{align*}
		F(\bw)
		\leq  F(\zero) +  \frac{\norm{\bw}^2 c^2}{2} -  \norm{\bw} \cdot M~.
	\end{align*}
	
	The above expression is smaller than $F(\zero)$ if $\norm{\bw} < \frac{2M}{c^2}$.

\section{Discussion on the Assumption on $b_\bv$}\label{appen:disc bv}

In \corollaryref{cor:rand int conv} we had an assumption that  $-\frac{b_\bv}{\norm{\tilde{\bv}}} \leq \alpha \cdot \frac{\sin\left(\frac{\pi}{8}\right)}{4}$. This implies that either the bias term $b_\bv$ is positive, or it is negative but not too large. Here we discuss why this assumption is crucial for the proof of the theorem, and what can we still say when this assumption does not hold. 

In \thmref{thm:negative2} we showed an example with $b_\bv < 0$ where gradient descent with random initialization does not converge w.h.p. to a global minimum even asymptotically\footnote{In \thmref{thm:negative2} we have $\norm{\bv} \neq 1$, but it still holds if we normalize $\bv$, namely, replace $\bv$ with $\frac{\bv}{\norm{\bv}}$.}. In the example from \thmref{thm:negative2} we have $-\frac{b_\bv}{\norm{\tilde{\bv}}}  = r \left(1 - \frac{1}{2d^2} \right)$, and the input distribution is uniform over a ball of radius $r$. In this case, we must choose $\alpha$ from Assumption \ref{assum:spread alpha beta} to be smaller than $r$ (otherwise $\beta=0$) and hence $-\frac{b_\bv}{\norm{\tilde{\bv}}} > \alpha \left(1 - \frac{1}{2d^2} \right)$. Therefore it does not satisfy the assumption $-\frac{b_\bv}{\norm{\tilde{\bv}}} \leq \alpha \cdot \frac{\sin\left(\frac{\pi}{8}\right)}{4}$(already for $d>1$). If we choose, e.g., $\alpha = \frac{r}{2}$, then the example from  \thmref{thm:negative2} satisfies $-\frac{b_\bv}{\norm{\tilde{\bv}}} = \alpha \left(2 - \frac{1}{d^2} \right) \leq 2 \alpha$.
It implies that our assumption on $-\frac{b_\bv}{\norm{\tilde{\bv}}}$ is tight up to a constant factor, and is also crucial for the proof, since already for $-\frac{b_\bv}{\norm{\tilde{\bv}}} =   2 \alpha$ we have an example of convergence to a non-global minimum. 

On the other hand, if $-\frac{b_\bv}{\norm{\tilde{\bv}}} > \alpha \cdot \frac{\sin\left(\frac{\pi}{8}\right)}{4}$ (i.e. the assumption does not hold) we can calculate the loss at zero:
\begin{align*}
	F(\zero) 
	&= \frac{1}{2} \cdot \E_\bx \left[ \left( \sigma(\bv^\top \bx) \right)^2\right]
	= \frac{1}{2} \cdot \E_\bx \left[ \onefunc(\tilde{\bv}^\top \tilde{\bx} + b_\bv \geq 0) \left(\tilde{\bv}^\top \tilde{\bx} + b_\bv \right)^2 \right]
	\\
	&= \frac{1}{2} \cdot \E_\bx \left[ \onefunc \left(\bar{\tilde{\bv}}^\top \tilde{\bx} \geq -\frac{b_\bv}{\norm{\tilde{\bv}}} \right) \norm{\tilde{\bv}}^2 \left( \bar{\tilde{\bv}}^\top \tilde{\bx} + \frac{b_\bv}{\norm{\tilde{\bv}}} \right)^2 \right]
	\\
	&\leq \frac{\norm{\tilde{\bv}}^2}{2} \cdot \E_\bx \left[ \onefunc \left(\bar{\tilde{\bv}}^\top \tilde{\bx} \geq \alpha \cdot \frac{\sin\left(\frac{\pi}{8}\right)}{4} \right)  \left( \bar{\tilde{\bv}}^\top \tilde{\bx} \right)^2 \right]~.
\end{align*}
Let $\epsilon>0$ be a small constant. Suppose that the distribution $\tilde{\cd}$ is spherically symmetric, and that $\alpha$ is large, such that the above expectation is smaller than $\epsilon$.
For such $\alpha$, we either have  $-\frac{b_\bv}{\norm{\tilde{\bv}}} \leq \alpha \cdot \frac{\sin\left(\frac{\pi}{8}\right)}{4}$, in which case gradient descent converges w.h.p. to the global minimum, or $-\frac{b_\bv}{\norm{\tilde{\bv}}} > \alpha \cdot \frac{\sin\left(\frac{\pi}{8}\right)}{4}$, in which case the loss at $\bw=\zero$ is already almost as good as the global minimum. For standard Gaussian distribution, we can choose $\alpha$ to be a large enough constant that depends only on $\epsilon$ (independent of the input dimension), hence $\beta$ will also be independent of $d$. This means that for standard Gaussian distribution, for every constant $\epsilon >0$ we can ensure either convergence to a global minimum, or the loss at $\zero$ is already $\epsilon$-optimal.

Note that in Remark \ref{rem:negative bias} we have shown another distribution which is non-symmetric and depends on the target $\bv$, such that the loss $F(\bm{0})$ is highly sub-optimal, but gradient flow converges to such a point with probability close to $\frac{1}{2}$.

\section{Proofs from \secref{sec: symmetric dist}}\label{appen:proof from symmetric dits}
 
Before proving \thmref{thm:symmetric dist}, we first proof two auxiliary propositions which bounds certain areas for which the vector $\bw$ cannot reach during the optimization process. The first proposition shows that if the norm of $\tilde{\bw}$ is small, and its bias is close to zero, then the bias must get larger. The second proposition shows that if the norm of $\tilde{\bw}$ is small, and the bias is negative, then the norm of $\tilde{\bw}$ must get larger.

\begin{proposition}\label{prop:bound bias grows}
Assume that $\norm{\tilde{\bw} - \tilde{\bv}}^2 \leq 1$, and that Assumption \ref{assum:main assum convergence} holds. If $\norm{\tilde{\bw}} \leq 0.4$ and $b_\bw\in \left[0,\frac{\alpha^3\beta}{640}\right] $ then $\left(\nabla F(\bw)\right)_{d+1} \leq -\frac{\alpha^3\beta}{640}$.

\end{proposition}

 \begin{proof}
The $d+1$ coordinate of the distribution $\mathcal{D}$ is a constant $1$. We denote by $\tilde{\Dcal}$ the first $d$ coordinates of the distribution $\mathcal{D}$. Hence, we can write:
 \begin{align}
     \left(\nabla F(\bw)\right)_{d+1} &= \mathbb{E}_{\bx \sim \mathcal{D}}\left[(\sigma(\bw^\top\bx) - \sigma(\bv^\top\bx))\mathbbm{1}(\bw^\top\bx >0)\right] \nonumber\\
     & = \mathbb{E}_{\tilde{\bx} \sim \tilde{\Dcal}}\left[(\sigma(\tilde{\bw}^\top\tilde{\bx} + b_\bw) - \sigma(\tilde{\bv}^\top\tilde{\bx} + b_\bv))\mathbbm{1}(\tilde{\bw}^\top\tilde{\bx} >-b_\bw)\right] \nonumber\\
     & = \mathbb{E}_{\tilde{\bx} \sim \tilde{\Dcal}}\left[(\tilde{\bw}^\top\tilde{\bx} + b_\bw)\cdot \mathbbm{1}(\tilde{\bw}^\top\tilde{\bx} >-b_\bw)\right] - \nonumber\\
     & -  \mathbb{E}_{\tilde{\bx} \sim \tilde{\Dcal}}\left[(\tilde{\bv}^\top\tilde{\bx} + b_\bv)\cdot \mathbbm{1}(\tilde{\bw}^\top\tilde{\bx} >-b_\bw, \tilde{\bv}^\top \tilde{\bx} > -b_\bv)\right]\label{eq:bias grad bound}
 \end{align}
 
 We will bound each term in \eqref{eq:bias grad bound} separately. Using the assumption that $\tilde{\Dcal}$ is spherically symmetric, we can assume w.l.o.g that $\tilde{\bw} = \norm{\tilde{\bw}}\be_1$, the first unit vector. Hence we have that :
 \begin{align}
     & \mathbb{E}_{\tilde{\bx} \sim \tilde{\Dcal}}\left[(\tilde{\bw}^\top\tilde{\bx} + b_\bw)\cdot \mathbbm{1}(\tilde{\bw}^\top\tilde{\bx} >-b_\bw)\right] \nonumber  \\
     = &\mathbb{E}_{\tilde{\bx} \sim \tilde{\Dcal}}\left[(\norm{\tilde{\bw}}x_1 + b_\bw)\cdot \mathbbm{1}\left(x_1 >-\frac{b_\bw}{\norm{\tilde{\bw}}}\right)\right] \nonumber\\
     =&\norm{\tilde{\bw}}\mathbb{E}_{\tilde{\bx} \sim \tilde{\Dcal}}\left[x_1 \mathbbm{1}\left(x_1 >-\frac{b_\bw}{\norm{\tilde{\bw}}}\right)\right] + b_\bw\mathbb{E}_{\tilde{\bx} \sim \tilde{\Dcal}}\left[\mathbbm{1}\left(x_1 >-\frac{b_\bw}{\norm{\tilde{\bw}}}\right)\right] \nonumber \\
     \overset{(a)}{\leq} & 0.4\mathbb{E}_{\tilde{\bx} \sim \tilde{\Dcal}}\left[x_1 \mathbbm{1}\left(x_1 >-\frac{b_\bw}{\norm{\tilde{\bw}}}\right)\right] + b_\bw \nonumber \\
      \overset{(b)}{\leq} &  0.4\mathbb{E}_{\tilde{\bx} \sim \tilde{\Dcal}}\left[x_1 \mathbbm{1}(x_1 >0)\right] + b_\bw \label{eq:bias grad bound first eq}~.
 \end{align}
Here, (a) is since $\norm{\tilde{\bw}} \leq 0.4$, and $\mathbb{E}_{\tilde{\bx} \sim \tilde{\Dcal}}\left[\mathbbm{1}\left(x_1 >-\frac{b_\bw}{\norm{\tilde{\bw}}}\right)\right] \leq 1$,  (b) is since $b_\bw \geq 0$, hence  
\[
\mathbb{E}_{\tilde{\bx} \sim \tilde{\Dcal}}\left[x_1 \mathbbm{1}\left(0 > x_1 >-\frac{b_\bw}{\norm{\tilde{\bw}}}\right)\right] \leq 0~.
\]

For the second term of \eqref{eq:bias grad bound}, we assumed that $\norm{\tilde{\bw} - \tilde{\bv}}^2 \leq 1$, which shows that $\theta(\tilde{\bw},\tilde{\bv}) \leq \frac{\pi}{2}$, and the term is largest when this angle is largest. Hence, to lower bound this term we can assume that $\theta(\tilde{\bw},\tilde{\bv}) = \frac{\pi}{2}$, and since the distribution is spherically symmetric we can also assume w.l.o.g that $\tilde{\bv}=\be_2$, the second unit vector. Now we can bound:
 
 \begin{align}
    &\mathbb{E}_{\tilde{\bx} \sim \tilde{\Dcal}}\left[(\tilde{\bv}^\top\tilde{\bx} + b_\bv)\cdot \mathbbm{1}(\tilde{\bw}^\top\tilde{\bx} >-b_\bw, \tilde{\bv}^\top \tilde{\bx} > -b_\bv)\right] \nonumber \\
    \geq & \mathbb{E}_{\tilde{\bx} \sim \tilde{\Dcal}}\left[(x_2 + b_\bv)\cdot \mathbbm{1}\left(x_1 >-\frac{b_\bw}{\norm{\tilde{\bw}}}, x_2 > -b_\bv\right)\right] \nonumber \\
    \geq & \frac{1}{2}\mathbb{E}_{\tilde{\bx} \sim \tilde{\Dcal}}\left[(x_2 + b_\bv)\cdot \mathbbm{1}\left(x_2 > -b_\bv\right)\right] \nonumber \\
    \geq & \frac{1}{2}\mathbb{E}_{\tilde{\bx} \sim \tilde{\Dcal}}\left[x_2\cdot \mathbbm{1}\left(x_2 >0\right)\right] + \frac{1}{2}\mathbb{E}_{\tilde{\bx} \sim \tilde{\Dcal}}\left[(x_2 + b_\bv)\cdot \mathbbm{1}\left( 0 > x_2 > -b_\bv\right)\right] \nonumber \\
    \geq & \frac{1}{2}\mathbb{E}_{\tilde{\bx} \sim \tilde{\Dcal}}\left[x_2\cdot \mathbbm{1}\left(x_2 >0\right)\right] = \frac{1}{2}\mathbb{E}_{\tilde{\bx} \sim \tilde{\Dcal}}\left[x_1\cdot \mathbbm{1}\left(x_1 >0\right)\right] ~,\label{eq:bias gradient bound second eq}
 \end{align}
where we used the assumption $b_\bv \geq 0$ and the symmetry of the distribution. Combining \eqref{eq:bias grad bound first eq}, \eqref{eq:bias gradient bound second eq} with \eqref{eq:bias grad bound} we get:

\begin{align*}
    \left(\nabla F(\bw)\right)_{d+1} \leq  b_\bw - 0.1\mathbb{E}_{\tilde{\bx} \sim \tilde{\Dcal}}\left[x_1 \cdot \mathbbm{1}(x_1 > 0)\right]~.
\end{align*}
Let $\hat{\mathcal{D}}$ be the marginal distribution of $\tilde{\Dcal}$ on the plane spanned by $\be_1$ and $\be_2$, and denote by $\hat{\bx}$ the projection of $\tilde{\bx}$ on this plane. By Assumption \ref{assum:main assum convergence}(3) we have that the pdf of this distribution is at least $\beta$ in a ball or radius $\alpha$ around the origin. This way we can bound:
\begin{align*}
    \mathbb{E}_{\tilde{\bx} \sim \tilde{\Dcal}}\left[x_1 \cdot \mathbbm{1}(x_1 > 0)\right] & = \mathbb{E}_{\hat{\bx} \sim \hat{\mathcal{D}}}\left[x_1 \cdot \mathbbm{1}(x_1 > 0)\right] \\ 
    \geq & \frac{\alpha\beta}{2}\text{P}(\alpha/2<\norm{\hat{\bx}}< \alpha,~ x_1 > \alpha / 2) \\
    \geq &\frac{\alpha\beta}{2}\text{P}(x_1\in [\alpha/2, 3\alpha/4],~ x_2\in [-\alpha/4, \alpha/4]) = \frac{\alpha^3\beta}{32}~.
\end{align*}
Combining the above, and using the assumption on $b_\bw$ we get that:
\begin{align*}
    \left(\nabla F(\bw)\right)_{d+1} \leq b_\bw - \frac{\alpha^3\beta}{320} \leq -\frac{\alpha^3\beta}{640}
\end{align*}

 \end{proof}

\begin{proposition}\label{prop: bound norm grows}
Assume that $\norm{\tilde{\bw} - \tilde{\bv}} < 1$, and Assumption \ref{assum:main assum convergence} holds. Denote by $\tau = \frac{\mathbb{E}_{\tilde{\bx} \sim \tilde{\Dcal}}\left[|x_1x_2|\right]}{\mathbb{E}_{\tilde{\bx} \sim \tilde{\Dcal}}\left[x_1^2\right]}$
 where $\tilde{\Dcal}$ is the projection of the distribution  $\mathcal{D}$ on its first $d$ coordinates.
If $\norm{\tilde{\bw}} \leq \frac{\tau}{2}$ and $b_\bw \leq 0$ then $\inner{\nabla F(\bw)_{1:d},\tilde{\bw}} \leq 0$.
\end{proposition}

\begin{proof}
Denote by $\tilde{\Dcal}$ the projection of the distribution  $\mathcal{D}$ on its first $d$ coordinates, we have that:
\begin{align}
    \inner{\nabla F(\bw)_{1:d},\tilde{\bw}}  = & \mathbb{E}_{{\bx} \sim \mathcal{D}}\left[\left(\sigma(\bw^\top \bx) - \sigma(\bv^\top \bx)\right) \mathbbm{1}(\bw^\top \bx > 0 ) \tilde{\bw}^\top \tilde{\bx}\right] \nonumber\\
     = & \mathbb{E}_{\tilde{\bx} \sim \tilde{\Dcal}}\left[\left(\sigma(\tilde{\bw}^\top \tilde{\bx} + b_\bw) - \sigma(\tilde{\bv}^\top \tilde{\bx} + b_\bv)\right) \mathbbm{1}(\tilde{\bw}^\top \tilde{\bx} > -b_\bw ) \tilde{\bw}^\top \tilde{\bx}\right] \nonumber \\
     \leq & \mathbb{E}_{\tilde{\bx} \sim \tilde{\Dcal}}\left[\left(\tilde{\bw}^\top \tilde{\bx} + b_\bw - \sigma(\tilde{\bv}^\top \tilde{\bx})\right) \cdot \mathbbm{1}(\tilde{\bw}^\top \tilde{\bx} > -b_\bw ) \tilde{\bw}^\top \tilde{\bx}\right]~.\label{eq:grad inner prod bound}
\end{align}
The inequality above is since $b_\bv \geq 0$. Recall that our goal is to prove that the above term is negative, hence we will divide it by $\norm{\tilde{\bw}}$. Also, since the distribution $\tilde{\Dcal}$ is symmetric we can assume w.l.o.g that $\tilde{\bw} = \norm{\tilde{\bw}}\be_1$.
Hence, it is enough to prove that the following term is non-positive:
\begin{align}
    & \norm{\tilde{\bw}}\mathbb{E}_{\tilde{\bx} \sim \tilde{\Dcal}}\left[\left(\norm{\tilde{\bw}}x_1 + b_\bw - \sigma(\tilde{\bv}^\top \tilde{\bx})\right) \cdot \mathbbm{1}\left(x_1 > -\frac{b_\bw}{\norm{\tilde{\bw}}} \right) x_1\right] \nonumber \\
    = & \norm{\tilde{\bw}}\mathbb{E}_{\tilde{\bx} \sim \tilde{\Dcal}}\left[\left(\norm{\tilde{\bw}}x_1^2 + b_\bw x_1\right) \cdot \mathbbm{1}\left(x_1 > -  \frac{b_\bw}{\norm{\tilde{\bw}}} \right)\right]  -  \norm{\tilde{\bw}}\mathbb{E}_{\tilde{\bx} \sim \tilde{\Dcal}}\left[x_1\tilde{\bv}^\top\tilde{\bx} \cdot \mathbbm{1}\left(x_1 > -  \frac{b_\bw}{\norm{\tilde{\bw}}}, \tilde{\bv}^\top \tilde{\bx} > -b_\bv \right)\right] \nonumber\\
    \leq & \norm{\tilde{\bw}}\mathbb{E}_{\tilde{\bx} \sim \tilde{\Dcal}}\left[\left(\norm{\tilde{\bw}}x_1^2 + b_\bw x_1\right) \cdot \mathbbm{1}\left(x_1 > -  \frac{b_\bw}{\norm{\tilde{\bw}}} \right)\right]  -  \norm{\tilde{\bw}}\mathbb{E}_{\tilde{\bx} \sim \tilde{\Dcal}}\left[x_1\tilde{\bv}^\top\tilde{\bx} \cdot \mathbbm{1}\left(x_1 > -  \frac{b_\bw}{\norm{\tilde{\bw}}}, \tilde{\bv}^\top \tilde{\bx} > 0 \right)\right]~.\label{eq:before bound v=e 2}
\end{align}
We will first bound the second term above. Since the term only depend on inner products between $\tilde{\bw}, \tilde{\bv}$ with $\tilde{\bx}$, we can consider the marginal distribution $\hat{\mathcal{D}}$, of $\tilde{\Dcal}$ on the plane spanned by $\tilde{\bw}$ and $\tilde{\bv}$. Since $\tilde{\Dcal}$ is symmetric we can assume w.l.o.g that $\hat{\mathcal{D}}$ is spanned by the first two coordinates $x_1$ and $x_2$. Let $\hat{\tilde{\bv}}$ be the projection of $\tilde{\bv}$ on this plane, then we can write $\hat{\tilde{\bv}} = (v_1,v_2)$ where $v_1^2 + v_2^2 = 1$. Note that since the distribution $\hat{\mathcal{D}}$ is symmetric, we have that  $\mathbb{E}[x_1^2] = \mathbb{E}[x_2^2]$. By Cauchy-Schwarz we have:
\[
|\text{cov}_{\hat{\mathcal{D}}}(x_1,x_2)| \leq \sqrt{\text{var}_{\hat{\mathcal{D}}}(x_1)\cdot \text{var}_{\hat{\mathcal{D}}}(x_2)} = \text{var}_{\hat{\mathcal{D}}}(x_1)
\]
Again, by symmetry of $\hat{\mathcal{D}}$ we have that $\mathbb{E}[x_1] = \mathbb{E}[x_2]$. Opening up the above terms we get that $\mathbb{E}[x_1\cdot x_2] \leq \mathbb{E}[x_1^2]$. Also, we assumed that $\norm{\tilde{\bw}-\tilde{\bv}} < 1$, then $\theta(\tilde{\bv},\tilde{\bw}) \leq \frac{\pi}{2}$ which means that $v_1 \geq 0$.  Hence, the second term of \eqref{eq:before bound v=e 2} is smallest when $\tilde{\bv} = \be_2$. In total, we can bound \eqref{eq:before bound v=e 2} by:
\begin{align}
     &\norm{\tilde{\bw}}\mathbb{E}_{\tilde{\bx} \sim \tilde{\Dcal}}\left[\left(\norm{\tilde{\bw}}x_1^2 + b_\bw x_1\right) \cdot \mathbbm{1}\left(x_1 > -  \frac{b_\bw}{\norm{\tilde{\bw}}} \right)\right]  -  \norm{\tilde{\bw}}\mathbb{E}_{\tilde{\bx} \sim \tilde{\Dcal}}\left[x_1x_2 \cdot \mathbbm{1}\left(x_1 > -  \frac{b_\bw}{\norm{\tilde{\bw}}}, x_2 > 0 \right)\right] \nonumber\\
     \leq  & \norm{\tilde{\bw}}\mathbb{E}_{\tilde{\bx} \sim \tilde{\Dcal}}\left[\left(\norm{\tilde{\bw}}x_1^2 + b_\bw x_1 -  \frac{1}{2} x_1|x_2| \right)\cdot \mathbbm{1}\left(x_1 > -  \frac{b_\bw}{\norm{\tilde{\bw}}} \right)\right] \nonumber\\
     = & \norm{\tilde{\bw}}\mathbb{E}_{\tilde{\bx} \sim \tilde{\Dcal}}\left[\left(\norm{\tilde{\bw}}x_1 + b_\bw -  \frac{1}{2} |x_2| \right)\cdot x_1\mathbbm{1}\left(x_1 > -  \frac{b_\bw}{\norm{\tilde{\bw}}} \right)\right]\label{eq:bound before b w =0}
\end{align}

By our assumption, $b_\bw \leq 0$. Both terms inside the expectation in \eqref{eq:bound before b w =0} are largest when $b_\bw =0$. Hence, we can bound \eqref{eq:bound before b w =0} by:
\begin{align}
    &\norm{\tilde{\bw}}\mathbb{E}_{\tilde{\bx} \sim \tilde{\Dcal}}\left[\left(\norm{\tilde{\bw}}x_1 -  \frac{1}{2} |x_2| \right)\cdot x_1\mathbbm{1}\left(x_1 > 0\right)\right] \nonumber\\
    = & \frac{\norm{\tilde{\bw}}^2}{2}\mathbb{E}_{\tilde{\bx} \sim \tilde{\Dcal}}\left[x_1^2\right] - \frac{\norm{\tilde{\bw}}}{4}\mathbb{E}_{\tilde{\bx} \sim \tilde{\Dcal}}\left[|x_1x_2|\right] \nonumber\\
    \leq & \frac{\norm{\tilde{\bw}}^2}{2}\mathbb{E}_{\tilde{\bx} \sim \tilde{\Dcal}}\left[x_1^2\right] - \frac{\norm{\tilde{\bw}}\tau}{4}\mathbb{E}_{\tilde{\bx} \sim \tilde{\Dcal}}\left[x_1^2\right] = c_1\left(\frac{\norm{\tilde{\bw}}^2}{2} - \frac{\norm{\tilde{\bw}}\tau}{4} \right)~.
    \label{eq:bound on norm w before tau}
\end{align}
In particular, for $\norm{\tilde{\bw}} \leq \frac{\tau}{2}$,  \eqref{eq:bound on norm w before tau} non-positive.

\end{proof}

We are now ready to prove the main theorem:
\begin{proof}[Proof of \thmref{thm:symmetric dist}]
Denote $b_t = \max\{0, -\frac{b_{\bw_t}}{\norm{\tilde{\bw}_t}}\}$. We will show by induction on the iterations of gradient descent that throughout the optimization process  $b_t < 2.4\cdot \max\left\{1,\frac{1}{\sqrt{\tau}}\right\}$ and $\theta(\tilde{\bw}_t,\tilde{\bv}) \leq \frac{\pi}{2}$ for every $t\geq 0$.

By the assumption on the initialization we have that $\norm{\tilde{\bw}_0 - \tilde{\bv}}^2 \leq \norm{\bw_0 - \bv}^2 <1$, and also $\norm{\tilde{\bv}} =1$, hence  $\theta(\tilde{\bw}_0,\tilde{\bv}) \leq \frac{\pi}{2}$. We also have that $b_{\bw_0} \geq 0$, hence $b_0 = 0$ this proves the case of $t=0$. Assume this is true for $t$. We will bound the norm of the gradient of the objective using Jensen's inequality:
\begin{align}
    \norm{\nabla F(\bw)}^2 & \leq \mathbb{E}_{\bx \sim \mathcal{D}} \left[(\sigma(\bw^\top \bx) - \sigma(\bv^\top \bx))^2 \mathbbm{1}(\bw^\top \bx >0)\bx^\top \bx\right]    \nonumber \\
    & \leq \mathbb{E}_{\bx \sim \mathcal{D}} \left[(\bw^\top \bx - \bv^\top \bx)^2 \bx^\top \bx\right] \nonumber \\
    & \leq \norm{\bw - \bv}^2\mathbb{E}_{\bx \sim \mathcal{D}} \left[\norm{\bx}^4\right] = \norm{\bw - \bv}^2c~. \label{eq:bound on norm of grad}
\end{align}
For the $(t+1)$-th iteration of gradient descent we have that:
\begin{align}\label{eq:bound on theta}
\norm{\bw_{t+1} - \bv}^2  =& \norm{\bw_t - \eta\nabla F(\bw_t) - \bv}^2  \nonumber\\
 = & \norm{\bw_t - \bv}^2 -2\eta\inner{\nabla F(\bw_t),\bw_t-\bv} + \eta^2 \norm{\nabla F(\bw_t)}^2 \nonumber\\
\leq & \norm{\bw_t - \bv}^2 - 2\eta\inner{\nabla F(\bw_t),\bw_t-\bv} + \eta^2 c\norm{\bw_t - \bv}^2~.
\end{align}
By \thmref{thm:innerprod}, and the induction assumption on $\theta(\tilde{\bw}_t,\tilde{\bv})$ we get that there is a universal constant $c_0$, such that $\inner{\nabla F(\bw_t),\bw_t-\bv} \geq \frac{c_0\beta(\alpha-\sqrt{2}b_t)}{\alpha^2}\norm{\bw_t - \bv}^2$. Using the induction assumption that $b_t < 2.4\cdot \max\left\{1,\frac{1}{\sqrt{\tau}}\right\}$ and Assumption \ref{assum:main assum convergence}(3) we can bound $(\alpha-\sqrt{2}b_t)\geq 0.1$. In total we get that $\inner{\nabla F(\bw_t),\bw_t-\bv} \geq \frac{c_0\beta}{10\alpha^2}\norm{\bw_t - \bv}^2$.
By taking $\eta \leq \frac{c_0\beta}{10c\alpha^2}$ and combining with \eqref{eq:bound on theta} we have that:
\[
\norm{\bw_{t+1} - \bv}^2 < \norm{\bw_{t} - \bv}^2~.
\]
In particular, $\norm{\tilde{\bw}_{t+1} - \tilde{\bv}}^2 \leq \norm{\bw_{t+1} - \bv}^2 <\norm{\bw_t - \bv}^2 <1$, which shows that $\theta(\tilde{\bw}_{t+1},\tilde{\bv}) \leq \frac{\pi}{2}$, and concludes the first part of the induction.

The bound for $b_t$ is more intricate, for an illustration see \figref{fig:bound_b_t}. Let $t'$ be the first iteration for which $\norm{\tilde{\bw}_{t'}} \geq 0.4$. First assume that $t\leq t'$, we will show that in this case $b_t = 0$. Assume otherwise, and let $t_0$ be the first iteration for which $b_{t_0} >0$, this means that $b_{\bw_{t_0}} <0$ and $b_{\bw_{t_0-1}} \geq 0$. We have that:
\begin{align*}
    b_{\bw_{t_0}} = b_{\bw_{t_0-1}} - \eta\nabla F(\bw_{t_0})_{d+1}~.
\end{align*}
If $b_{\bw_{t_0-1}} \leq \frac{\alpha^3\beta}{640}$, then by \propref{prop:bound bias grows} the last coordinate of the gradient is negative, hence
$b_{\bw_{t_0}} > b_{\bw_{t_0-1}} \geq 0$. Otherwise, assume that $b_{\bw_{t_0-1}} > \frac{\alpha^3\beta}{640}$. By \eqref{eq:bound on norm of grad}: $\left|\nabla F(\bw_{t_0})_{d+1}\right| \leq \norm{\nabla F(\bw)} \leq \sqrt{c}$. Hence, by taking $\eta < \frac{\beta}{640\sqrt{c}} \leq \frac{\alpha^3\beta}{640\sqrt{c}}$, we get that $ b_{\bw_{t_0}} \geq 0$, which is a contradiction (note that by Assumption \ref{assum:main assum convergence}(3), we have $\alpha \geq 1$). We proved that if $t\leq t'$ then $b_{\bw_t} \geq 0$, which means that $b_t = 0$.

Assume now that $t > t'$. We will need the following calculation: Assume that $\norm{\tilde{\bw}_t} = \delta$, Then $\norm{\tilde{\bw}_t - \tilde{\bv}}^2 \geq (1-\delta)^2$, and the minimum is achieved at $\tilde{\bw} = \delta \tilde{\bv}$. Since we have:
\[
\norm{\tilde{\bw}_t - \tilde{\bv}}^2 + (b_{\bw_t} - b_\bv)^2 = \norm{\bw_t-\bv}^2 \leq 1~,
\]
we get that $(b_{\bw_t} - b_\bv)^2 \leq 1 - (1- \delta)^2 \leq 2\delta$. If we further assume that $b_{\bw_t} \leq 0 $, then $b_{\bw_t}^2 \leq (b_{\bw_t} - b_\bv)^2 \leq 2\delta$. Combining all the above, we get that if $\norm{\tilde{\bw}_t} = \delta$ then:
\begin{equation}\label{eq:general bound on b}
    b_t = \max\left\{0,-\frac{b_{\bw_t}}{\norm{\tilde{\bw}_t}}\right\} \leq \sqrt{\frac{2}{\delta}}~.
\end{equation}
To show the bound on $b_t$ we split into cases, depending on the norm of $\tilde{\bw}_t$:

\textbf{Case I:} $\frac{2\tau}{5} < \norm{\tilde{\bw}_t} \leq \frac{\tau}{2}$ and $b_{\bw_t} \leq 0$. In this case we have:

\begin{align*}
\norm{\tilde{\bw}_{t+1}}^2 &= \norm{\tilde{\bw}_t - \eta \nabla F(\bw_t)_{1:d} }^2 \\
& = \norm{\tilde{\bw}_t}^2 - 2\eta\inner{\tilde{\bw}_t, \nabla F(\bw_t)_{1:d}} + \eta^2\norm{ \nabla F(\bw_t)_{1:d}}^2 \\
& \geq \norm{\tilde{\bw}_t}^2 - 2\eta\inner{\tilde{\bw}_t, \nabla F(\bw_t)_{1:d}}~.
\end{align*}

We can use \propref{prop: bound norm grows} to get that $\inner{\tilde{\bw}_t, \nabla F(\bw_t)_{1:d}} \leq 0$, hence $\norm{\tilde{\bw}_{t+1}}^2  \geq \norm{\tilde{\bw}_t}^2$. By \eqref{eq:general bound on b} we get that $b_{t+1} \leq \sqrt{\frac{5}{\tau}} \leq \frac{2.4}{\sqrt{\tau}}$.


\textbf{Case II:} $\norm{\tilde{\bw}_t} \geq \min\left\{0.4,\frac{\tau}{2}\right\}$. In this case, by choosing a step size $\eta < \frac{1}{40c}\min\{1,\tau\} $we can bound 
\begin{align*}
    \norm{\tilde{\bw}_{t+1}} &\geq \norm{\tilde{\bw}_t}^2 - 2\eta\inner{\tilde{\bw}_t, \nabla F(\bw_t)_{1:d}} \\
    & \geq \norm{\tilde{\bw}_t}^2 -2\eta \norm{\tilde{\bw}_t}\norm{\nabla F(\bw_t)_{1:d}} \\
    & \geq \norm{\tilde{\bw}_t}^2 -2\eta \norm{\tilde{\bw}_t}\norm{\nabla F(\bw_t)} \\
    & \geq \norm{\tilde{\bw}_t}^2 -2\eta \cdot 2c \geq  \min\left\{0.39, \frac{2\tau}{5}\right\}~.
\end{align*}
Again, by \eqref{eq:general bound on b} we get that $b_{t+1} \leq \max\left\{\sqrt{5.2}, \frac{2.4}{\sqrt{\tau}}  \right\} \leq 2.4\cdot  \max\left\{1, \frac{1}{\sqrt{\tau}}  \right\}$. This concludes the induction.

\textbf{Case III:} $\norm{\tilde{\bw}_t} \leq \min\left\{0.4, \frac{2\tau}{5} \right\}$. We split into sub-cases depending on the previous iteration: (a) If $b_{\bw_{t-1}} \leq 0$, then by Case I the norm of $\tilde{\bw}$ cannot get below $\frac{2\tau}{5}$, hence this sub-case is not possible; (b) If  $b_{\bw_{t-1}} \geq 0$ and $\norm{\tilde{\bw}_{t-1}} \leq \min\left\{0.4, \frac{2\tau}{5} \right\}$, then by the same reasoning in the case of $t<t'$, $b_{\bw_t}$ cannot get smaller than zero. Hence, we must have that $b_{\bw_{t+1}} \geq 0$; 
(c) If  $b_{\bw_{t-1}} \geq 0$ and $\norm{\tilde{\bw}_{t-1}} \geq \min\left\{0.4, \frac{2\tau}{5} \right\}$ then the bound depend on whether $\norm{\tilde{\bw}_{t-1}}$ is larger than $0.4$ or not. If $\norm{\tilde{\bw_{t-1}}}\leq 0.4$, then using the same reasoning as the case of $t'<t$ twice (both for the $t-1$ and $t$ iterations) we get that $b_{t+1} \geq 0$. If $\norm{\tilde{\bw}_{t-1}} > 0.4$ and $b_{\bw_t} \geq 0$, then again this is the same case as in the case of $t'<t$ (since $\norm{\tilde{\bw}_t} \leq 0.4$. The last case is when $\norm{\tilde{\bw}_{t-1}} > 0.4$ and $b_{\bw_t} <0$, here using the same calculation as in Case II, we have that $\norm{\tilde{\bw}_t} \geq 0.39$. Since $\norm{\tilde{\bw}_t} \leq \min\left\{0.4, \frac{2\tau}{5} \right\}$, using \propref{prop: bound norm grows}, the norm of $\tilde{\bw}_t$ can only grow, hence by the same reasoning as in Case I we can also bound $b_{t+1}< 2.4\max\left\{1,\frac{1}{\sqrt{\tau}} \right\}$.




\begin{figure}
    \centering
    \includegraphics[width=3.5in]{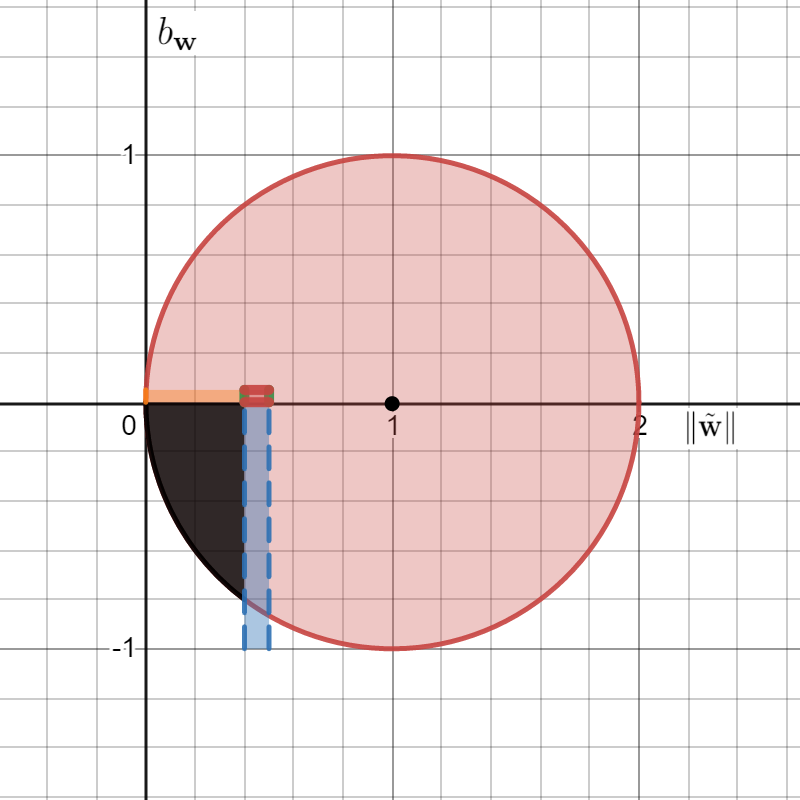}    \caption{A 2-d illustration of the optimization landscape. The $x$ axis represents $\norm{\tilde{\bw}}$, and the $y$-axis represents $b_\bw$. In the figure, for simplicity, we assume that $b_\bv=0$, and $\tau=0.1$ which means that $\frac{2\tau}{5}=0.4$. The red circle represents the area with $\norm{\bw-\bv} \leq 1$, throughout the optimization process $\bw_t$ stays in this circle.
    The black region represents the area where $b_t = -\frac{b_\bw}{\norm{\tilde{\bw}}}$ can be potentially large, our goal is to show that $\bw_t$ stays out of this region. Case I shows that $\bw_t$ cannot cross the blue region. Case II shows that if $\bw_t$ is to the right of the black region, then $b_t$ is upper bounded. Case III shows that $\bw_t$ cannot cross the orange region (sub-cases (a) and (b)), and cannot cross from the green region directly to the black region (sub-case (c)).}
    \label{fig:bound_b_t}
\end{figure}

Until now we have proven that throughout the entire optimization process we have that $\theta(\tilde{\bw}_{t},\tilde{\bv}) \leq \frac{\pi}{2}$ and $b_t \leq 2.4\cdot  \max\left\{1, \frac{1}{\sqrt{\tau}}  \right\}$. Let $\delta = \pi - \theta(\tilde{\bw}_{t},\tilde{\bv})$, we now use \thmref{thm:innerprod} and \eqref{eq:bound on norm of grad} to get that:
\begin{align}
    \norm{\bw_{t+1} - \bv}^2  =& \norm{\bw_t - \eta\nabla F(\bw_t) - \bv}^2  \nonumber\\
    = & \norm{\bw_t - \bv}^2 -2\eta\inner{\nabla F(\bw_t),\bw_t-\bv} + \eta^2 \norm{\nabla F(\bw_t)}^2 \nonumber\\
    \leq & \norm{\bw_t - \bv}^2 - 2\eta\frac{\left(\alpha - \frac{b_t}{\sin\left(\frac{\delta}{2}\right)}\right)^4\beta}{8^4\alpha^2}\sin\left(\frac{\delta}{4}\right)^3\norm{\bw_t - \bv}^2 + \eta^2c\norm{\bw_t - \bv}^2 \nonumber \\
    \leq & \norm{\bw_t - \bv}^2 - \eta\frac{\left(\alpha - \sqrt{2}b_t\right)^4\beta}{8^4\alpha^2}\sin\left(\frac{\delta}{4}\right)^3\norm{\bw_t - \bv}^2 + \eta^2c\norm{\bw_t - \bv}^2 \nonumber \\
    \leq & \norm{\bw_t - \bv}^2 - \frac{\eta\tilde{C}\beta}{\alpha^2}\norm{\bw_t - \bv}^2 + \eta^2c\norm{\bw_t - \bv}^2
\end{align}
where $\tilde{C}$ is some universal constant, and we used the bounds from the induction above that $\delta \in \left[\frac{\pi}{2},\pi\right]$, $b_t \leq 2.4\cdot \max\left\{1, \frac{1}{\sqrt{\tau}}  \right\}$, and by the assumption that $\alpha \geq 2.5\sqrt{2}\max\left\{1, \frac{1}{\sqrt{\tau}}  \right\}$. By choosing $\eta \leq \frac{\tilde{C}\beta}{2c\alpha^2}$, and setting $\lambda = \frac{\tilde{C}\beta}{2c\alpha^2}$ we get that:
\begin{align*}
& \norm{\bw_t - \bv}^2 - \eta\tilde{C}\beta\min\left\{1,\frac{1}{\alpha^2}\right\}\norm{\bw_t - \bv}^2 + \eta^2c\norm{\bw_t - \bv}^2    \\
\leq & (1-\lambda\eta)\norm{\bw_t - \bv}^2 \leq \dots \leq (1-\eta\lambda)^{t}\norm{\bw_0 - \bv}^2 ~,
\end{align*}
which finished the proof.

\end{proof}

\end{document}